\documentclass{article}

\usepackage{microtype}
\usepackage{graphicx}
\usepackage{booktabs} 

\usepackage{hyperref}

\usepackage[accepted]{icml2023}

\usepackage{amsmath}
\usepackage{amssymb}
\usepackage{mathtools}
\usepackage{amsthm}

\usepackage[capitalize,noabbrev]{cleveref}

\theoremstyle{plain}

\newtheorem{lemma}{Lemma}

\theoremstyle{definition}

\theoremstyle{remark}

\usepackage[textsize=tiny]{todonotes}

\icmltitlerunning{Byzantine-Robust Learning on Heterogeneous Data via Gradient Splitting}

\usepackage{amsthm}
\usepackage{booktabs}
\usepackage{caption}
\usepackage{graphicx}
\usepackage{makecell}
\usepackage{multirow}
\usepackage{scrextend}
\usepackage{subcaption}
\usepackage{thmtools}
\usepackage{thm-restate}
\usepackage{wrapfig}
\usepackage{xparse}
\usepackage{xspace}
\usepackage{layouts}

\crefname{ineq}{Inequality}{Inequalities}
\crefname{assumption}{Assumption}{Assumptions}
\crefname{proposition}{Proposition}{Propositions}

\creflabelformat{ineq}{#2\textup{(#1)}#3}

\newcommand{\scalefactor}{0.88}

\usepackage{amsmath,amsfonts,bm}

\def\eqref#1{equation~\ref{#1}}

\def\1{\bm{1}}

\def\vg{{\bm{g}}}

\def\vw{{\bm{w}}}
\def\vx{{\bm{x}}}

\DeclareMathAlphabet{\mathsfit}{\encodingdefault}{\sfdefault}{m}{sl}
\SetMathAlphabet{\mathsfit}{bold}{\encodingdefault}{\sfdefault}{bx}{n}

\def\gB{{\mathcal{B}}}

\def\gH{{\mathcal{H}}}
\def\gI{{\mathcal{I}}}

\def\gO{{\mathcal{O}}}

\def\gS{{\mathcal{S}}}

\def\sN{{\mathbb{N}}}

\def\sR{{\mathbb{R}}}

\newcommand{\E}{\mathbb{E}}

\newcommand{\Var}{\mathrm{Var}}

\newcommand{\vdelta}{\bm{\delta}}

\newcommand{\vX}{\bm{X}}
\newcommand{\vxi}{\bm{\xi}}

\newcommand{\nclients}{n}
\newcommand{\nattackers}{f}
\newcommand{\honestclients}{\mathcal{H}}
\newcommand{\byzantineclients}{\mathcal{B}}
\newcommand{\sample}{\vxi}
\newcommand{\samples}[1]{\vxi_{#1}}
\newcommand{\grad}[2]{\bar{\vg}_{#1}^{#2}}
\newcommand{\sgrad}[2]{\vg_{#1}^{#2}}
\newcommand{\avgsgrad}[1]{\vg^{#1}}
\newcommand{\aggsgrad}[1]{\hat{\vg}^{#1}}
\newcommand{\subgrad}[2]{\bar{\vg}_{#1}^{(#2)}}
\newcommand{\subsgrad}[2]{\vg_{#1}^{(#2)}}
\newcommand{\subavgsgrad}[1]{\vg^{(#1)}}
\newcommand{\subaggsgrad}[1]{\hat{\vg}^{(#1)}}
\newcommand{\param}[1][]{\vw^{#1}}
\newcommand{\localparam}[2]{\vw_{#1}^{#2}}
\newcommand{\nparams}{d}

\newcommand{\loss}[1][]{\mathcal{L}_{#1}}
\newcommand{\batch}[2]{\vxi_{#1}^{#2}}
\newcommand{\agg}{\mathcal{A}}

\newcommand{\proposedmethod}{GAS\xspace}
\newcommand{\shorten}{GrAdient Splitting\xspace}
\newcommand{\ngroups}{p}
\newcommand{\indexset}[1]{\mathcal{J}_{#1}}
\newcommand{\component}[2]{[#1]_{#2}}
\newcommand{\clientscore}[1]{s_{#1}}
\newcommand{\groupscore}[2]{s_{#1}^{(#2)}}
\newcommand{\target}{y}

\begin{document}

\twocolumn[
\icmltitle{Byzantine-Robust Learning on Heterogeneous Data via Gradient Splitting}



\icmlsetsymbol{equal}{*}
\icmlsetsymbol{intern}{\dag}

\begin{icmlauthorlist}
\icmlauthor{Yuchen Liu}{zju,equal,intern}
\icmlauthor{Chen Chen}{sony,equal}
\icmlauthor{Lingjuan Lyu}{sony}
\icmlauthor{Fangzhao Wu}{microsoft}
\icmlauthor{Sai Wu}{zju}
\icmlauthor{Gang Chen}{zju}
\end{icmlauthorlist}

\icmlaffiliation{zju}{Key Lab of Intelligent Computing Based Big Data of Zhejiang Province, Zhejiang University, Hangzhou, China}
\icmlaffiliation{sony}{Sony AI}
\icmlaffiliation{microsoft}{Microsoft}

\icmlcorrespondingauthor{Lingjuan Lyu}{lingjuan.lv@sony.com}

\icmlkeywords{Machine Learning, ICML}

\vskip 0.3in
]

\printAffiliationsAndNotice{\icmlEqualContribution \icmlIntern} 

\begin{abstract}
Federated learning has exhibited vulnerabilities to Byzantine attacks, where the Byzantine attackers can send arbitrary gradients to a central server to destroy the convergence and performance of the global model.
A wealth of robust AGgregation Rules (AGRs) have been proposed to defend against Byzantine attacks.
However, Byzantine clients can still circumvent robust AGRs when data is non-Identically and Independently Distributed (non-IID).
In this paper, we first reveal the root causes of performance degradation of current robust AGRs in non-IID settings: the curse of dimensionality and gradient heterogeneity.
In order to address this issue, we propose \proposedmethod, a \shorten approach that can successfully adapt existing robust AGRs to non-IID settings. 
We also provide a detailed convergence analysis when the existing robust AGRs are combined with \proposedmethod.
Experiments on various real-world datasets verify the efficacy of our proposed \proposedmethod.
The implementation code is provided in \url{https://github.com/YuchenLiu-a/byzantine-gas}.
\end{abstract}

\section{Introduction}
\label{sec:intro}

Federated Learning (FL) \cite{mcmahan2017fl,lyu2020threats,zhao2020privacy} provides a privacy-aware and distributed machine learning paradigm.
It has recently attracted widespread attention as a result of emerging data silos and growing privacy awareness.
In this paradigm, data owners (clients) repeatedly use their private data to compute local gradients and send them to a central server for aggregation.
In this way, clients can collaborate to train a model without exposing their private data.
However, the distributed property of FL also makes it vulnerable to Byzantine attacks \cite{blanchard2017krum, guerraoui2018bulyan,chen2020robust}, in which Byzantine clients can send arbitrary messages to the central server to bias the global model.
Moreover, it is challenging for the server to identify the Byzantine clients, since the server can neither access clients' training data nor monitor their local training process.

In order to defend against Byzantine attacks, the community has proposed a wealth of defenses \cite{blanchard2017krum, guerraoui2018bulyan, yin2018mediantrmean}.
Most defenses abandon the averaging step adopted by conventional FL frameworks, e.g., FedAvg \cite{mcmahan2017fl}.
Instead, they use robust AGgregation Rules (AGRs) to aggregate local gradients and compute the global gradient.
Most existing robust AGRs assume that the data distribution on different clients is Identically and Independently Distributed (IID) \cite{bernstein2018signsgd, ghosh2019robust}.
In fact, the data is usually heterogeneous, i.e., non-IID, in real-world FL applications \cite{mcmahan2017fl, kairouz2021advances,lyu2022privacy,zhang2023delving,Chen22CalFAT}.
In this paper, we focus on defending against Byzantine attacks in the more realistic non-IID setting.

In the non-IID setting, defending against Byzantine attacks becomes more challenging \cite{karimireddy2022bucketing, acharya2022cgm}.
Robust AGRs that try to include \emph{all} the honest gradients in aggregation \cite{blanchard2017krum,shejwalkar2021dnc} fail to handle the curse of dimensionality \cite{guerraoui2018bulyan}.
Byzantine clients can take advantage of the high dimension of gradients and participate in aggregation.
As a result, the global gradient is manipulated away from the optimal gradient, i.e., the average of honest gradients.
Other robust AGRs \cite{guerraoui2018bulyan,yin2018mediantrmean} aggregate \emph{fewer} gradients to ensure that only honest gradients participate in aggregation.
However, the global gradient is still of limited utility due to gradient heterogeneity \cite{li2020fedprox,karimireddy2020scaffold} in the non-IID setting.
In summary, most existing AGRs fail to address both the curse of dimensionality \cite{guerraoui2018bulyan} and gradient heterogeneity \cite{karimireddy2022bucketing} at the same time.
Consequently, they fail to achieve satisfactory performance in the non-IID setting.

Motivated by the above observations, we propose a \shorten based approach called \proposedmethod for Byzantine robustness in non-IID settings. 
In particular, to address the curse of dimensionality, \proposedmethod splits each high-dimensional gradient into low-dimensional sub-vectors and detects Byzantine gradients with the sub-vectors.
To handle the gradient heterogeneity, \proposedmethod aggregates all the identified honest gradients.

Our contributions in this work are summarized below.
\begin{itemize}
    \item We reveal the root causes of defending against Byzantine attacks in the non-IID setting: the gradient heterogeneity and the curse of dimensionality. 
    Gradient heterogeneity makes it hard for Byzantine defenses to obtain a global gradient close to the optimal.
    The curse of dimensionality enables the Byzantine gradients to circumvent defenses that aggregate more gradients. 
    To the best of our knowledge, no existing defense can address both issues at the same time.  
    \item We propose a novel and compatible approach called \proposedmethod which consists of three steps: 
    1. splitting the high-dimensional gradients into low-dimensional sub-vectors;
    2. penalizing each gradient by a score with a robust AGR based on their split low-dimensional sub-vectors to circumvent the curse of dimensionality;
    3. identifying the gradients with low scores as honest ones and aggregating all the identified honest gradients to tackle the gradient heterogeneity issue. 
    In step 2, \proposedmethod can apply any robust AGR to low-dimensional sub-vectors for identification, offering great compatibility. 
    \item We provide convergence analysis for our proposed \proposedmethod.
     Extensive experiments on four real-world datasets across various non-IID settings empirically validate the effectiveness and superiority of our \proposedmethod.
\end{itemize}

\section{Related Works}
\label{sec:related_works}

\textbf{IID defenses.}
\citet{blanchard2017krum} first introduce Byzantine robust learning and propose a distance-based AGR called Multi-Krum.
\citet{yin2018mediantrmean} theoretically analyze the statistical optimality of Median and Trimmed Mean.
\citet{guerraoui2018bulyan} propsoe Bulyan that applies a variant of Trimmed Mean as a post-processing method to handle the curse of dimensionality. 
\citet{pillutla2019geometric} discuss the Byzantine robustness of Geometric Median and propose a computationally efficient approximation of Geometric Median. 
\citet{shejwalkar2021dnc} propose to perform dimensionality reduction using random sampling, followed by spectral-based outlier removal. 
These defenses assume the data is IID.
Their efficacy is therefore limited in more realistic FL applications where the data is non-IID. 

\textbf{Non-IID defenses.}
Recent works have also explored defenses applicable to the non-IID setting.
\citet{park2021sageflow} can only achieve Byzantine robustness when the server has a validation set, which compromises the privacy principle of the FL \cite{mcmahan2017fl}.
\citet{data2021byzantine} adapt a robust mean estimation approach to FL in order to combat the Byzantine attack in the non-IID setting.
However, it requires $\Omega(d^2)$ time ($d$ is the number of model parameters), which is unacceptable due to the high dimensionality of model parameters.
\citet{el2021mda} consider Byzantine robustness in the asynchronous communication and unconstrained topologies settings.
\citet{acharya2022cgm} propose to apply geometric median only to the sparsified gradients to save computation cost.
\citet{karimireddy2022bucketing} perform a bucketing process before aggregation to reduce the gradient heterogeneity.
However, most of these methods ignore the curse of dimensionality \cite{guerraoui2018bulyan}, which becomes intractable in the non-IID setting (refer to \cref{sec:motivation} for more discussion).
As a result, they fail to achieve satisfactory performance in the non-IID setting.

\begin{figure*}[t]
\centering
\includegraphics[width=0.6\textwidth]{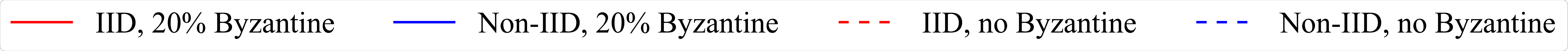}
\\
\begin{subfigure}[t]{0.48\textwidth}
\centering
\includegraphics[width=0.32\textwidth]{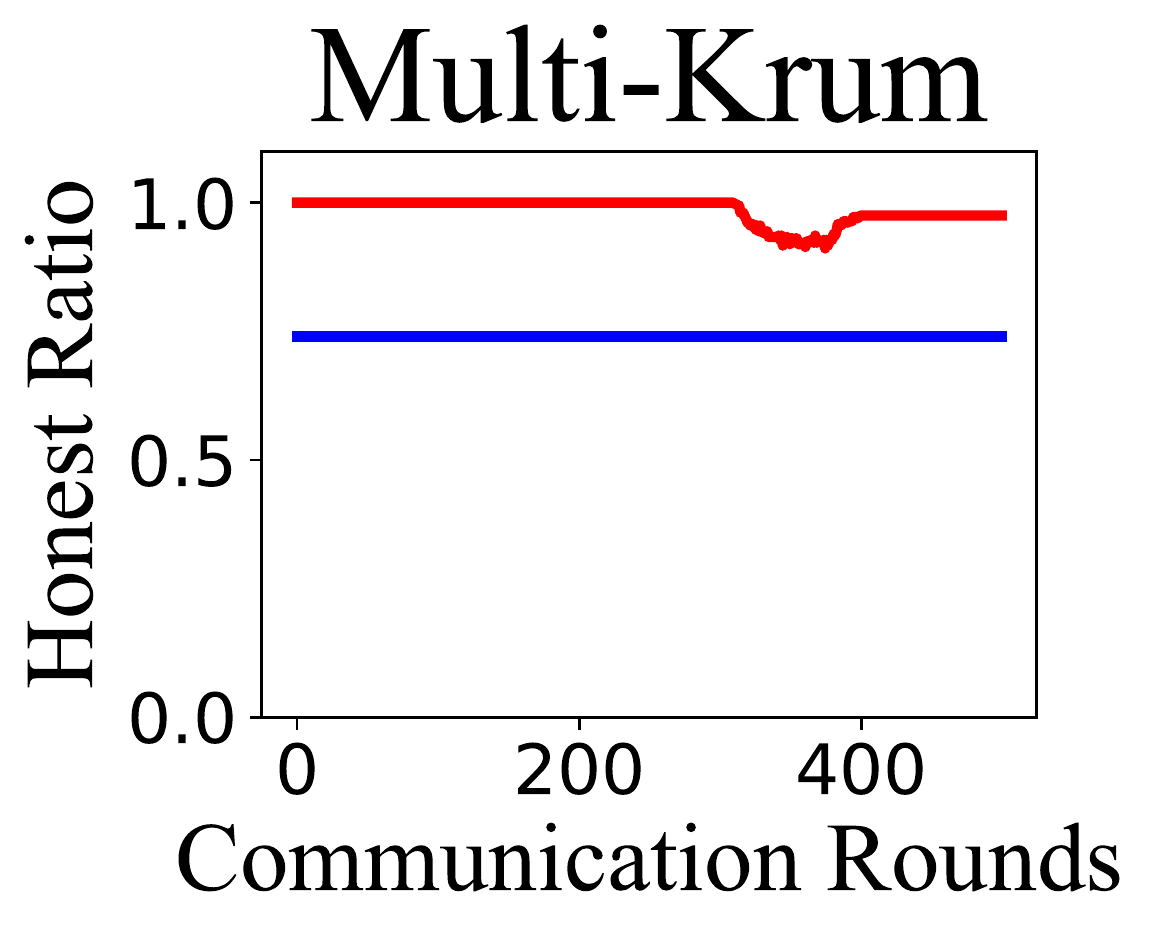}
\includegraphics[width=0.32\textwidth]{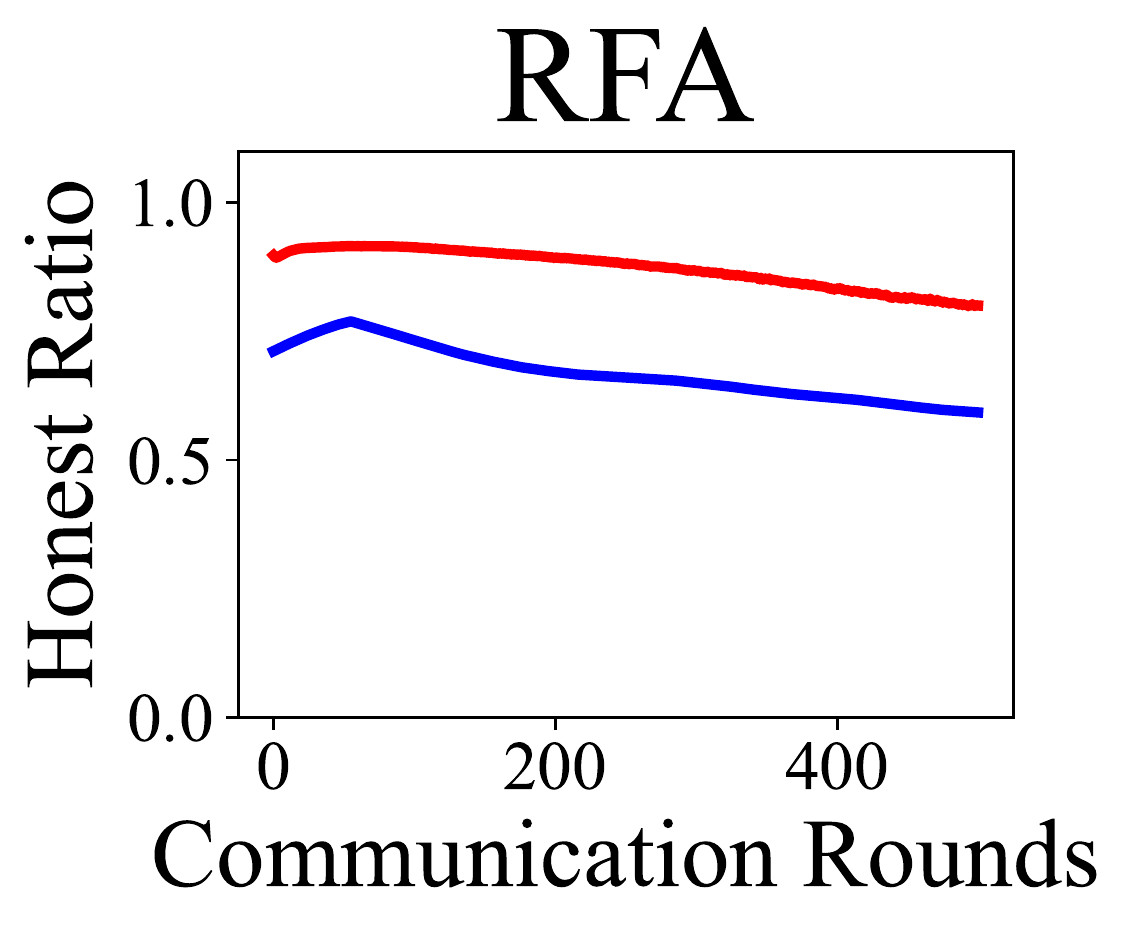}
\includegraphics[width=0.32\textwidth]{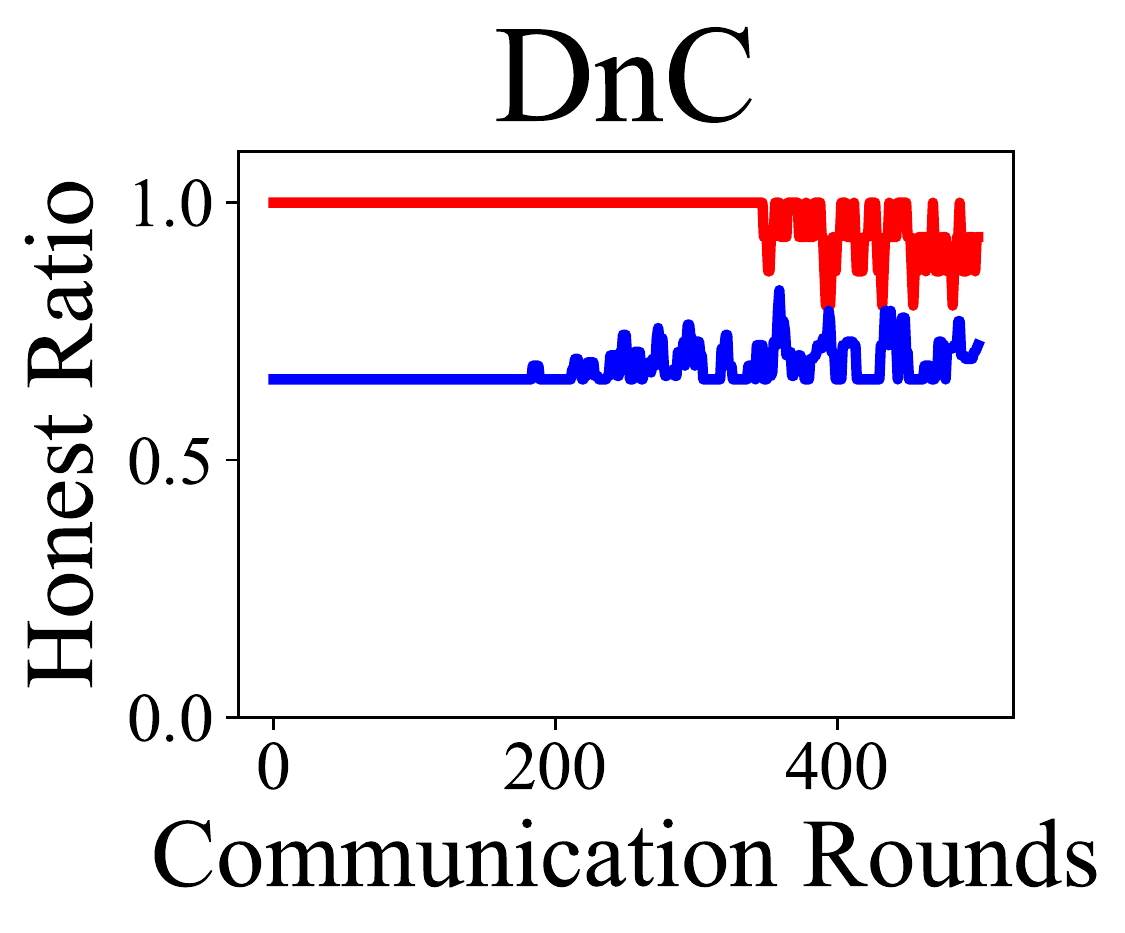}
\caption{The ratio of honest gradients that participate in the aggregation for robust AGRs that try to include all honest gradients (Multi-Krum \cite{blanchard2017krum}, RFA \cite{pillutla2019geometric}, DnC \cite{shejwalkar2021dnc}).}
\label{fig:more_selection}
\end{subfigure}
\hfill
\begin{subfigure}[t]{0.48\textwidth}
\centering
\includegraphics[width=0.32\textwidth]{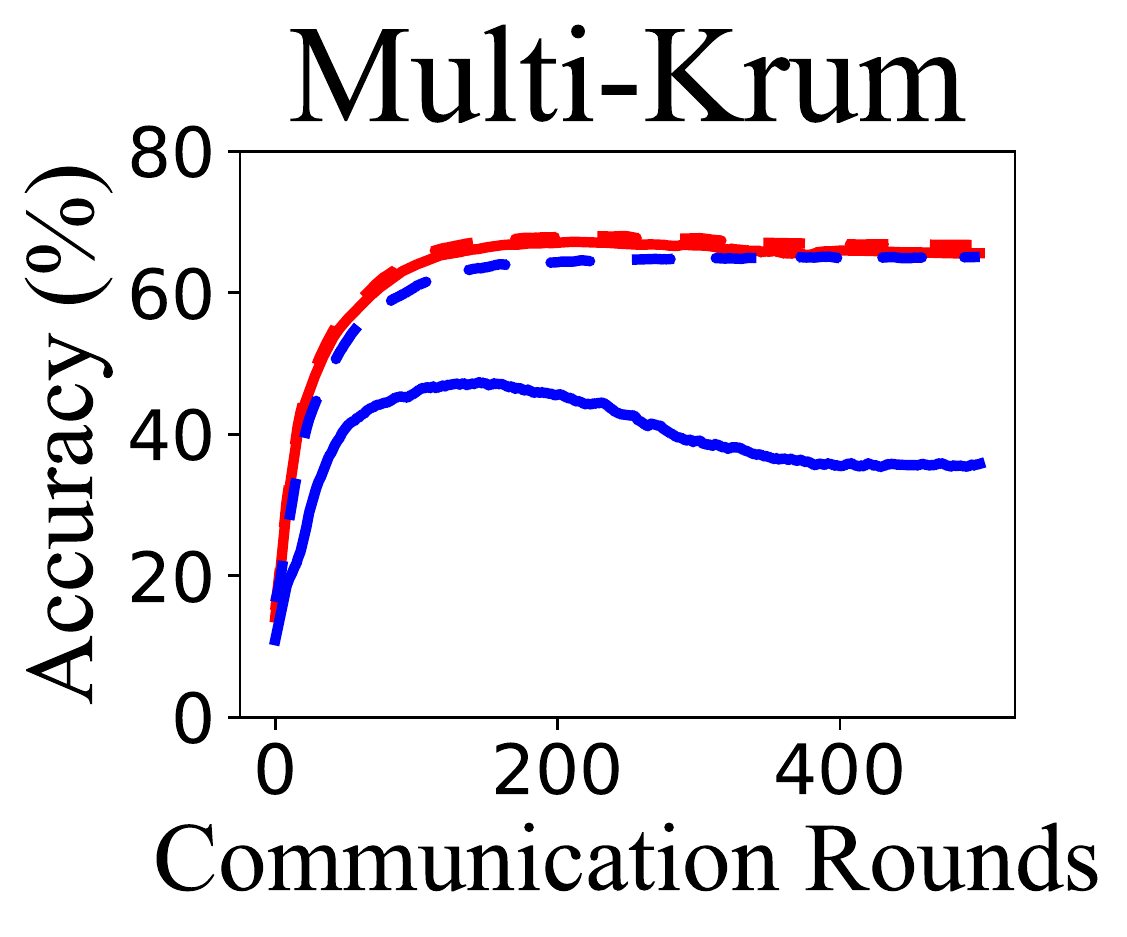}
\includegraphics[width=0.32\textwidth]{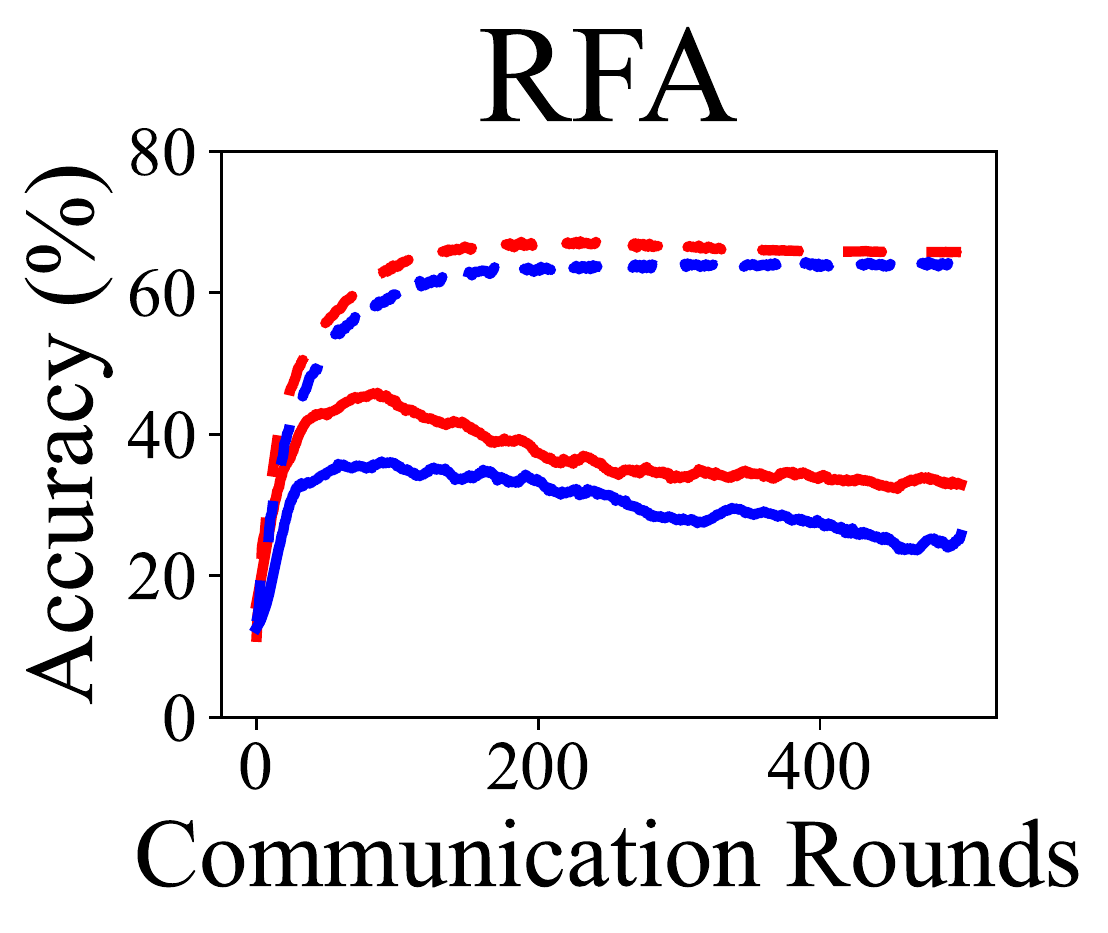}
\includegraphics[width=0.32\textwidth]{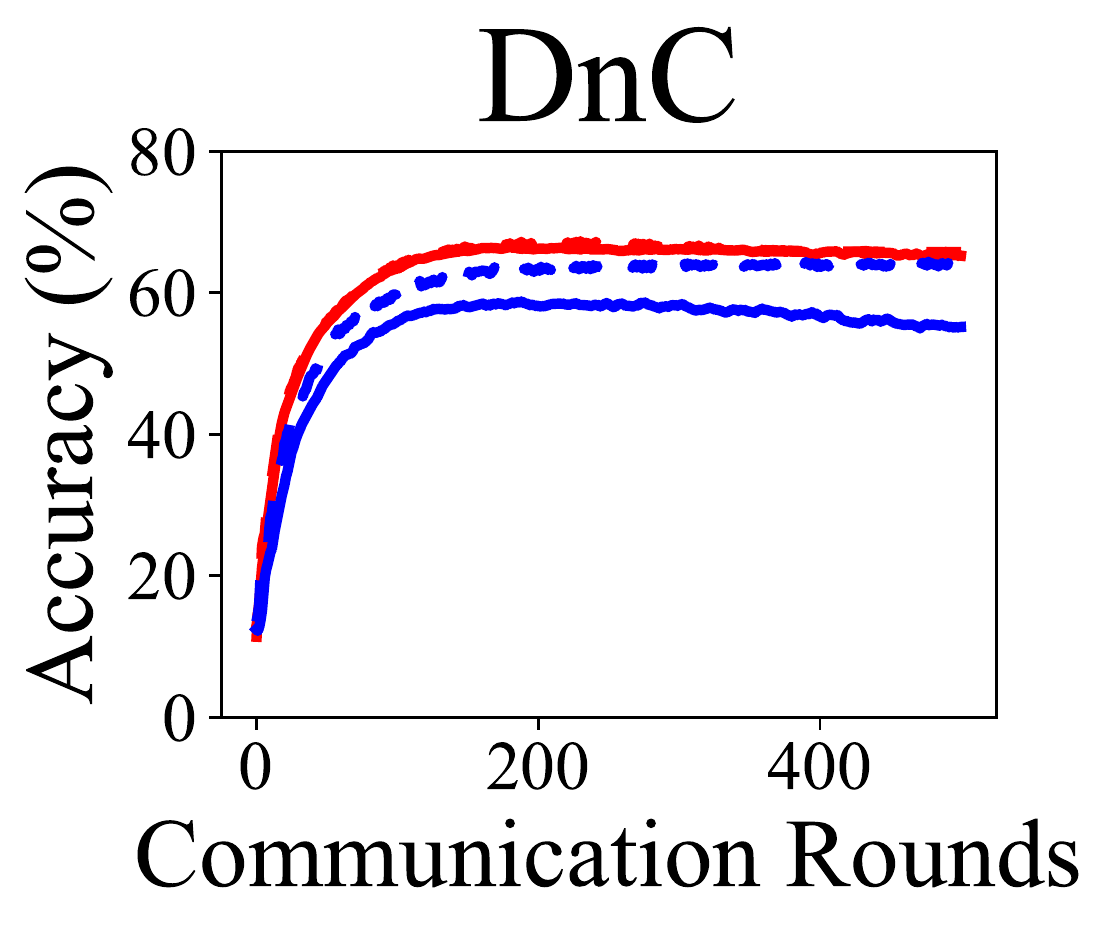}
\caption{The test accuracy of robust AGRs that try to include all honest gradients (Multi-Krum \cite{blanchard2017krum}, RFA \cite{pillutla2019geometric}, DnC \cite{shejwalkar2021dnc}).}
\label{fig:more_performance}
\end{subfigure}
\\
\begin{subfigure}[t]{0.48\textwidth}
\centering
\includegraphics[width=0.31\textwidth]{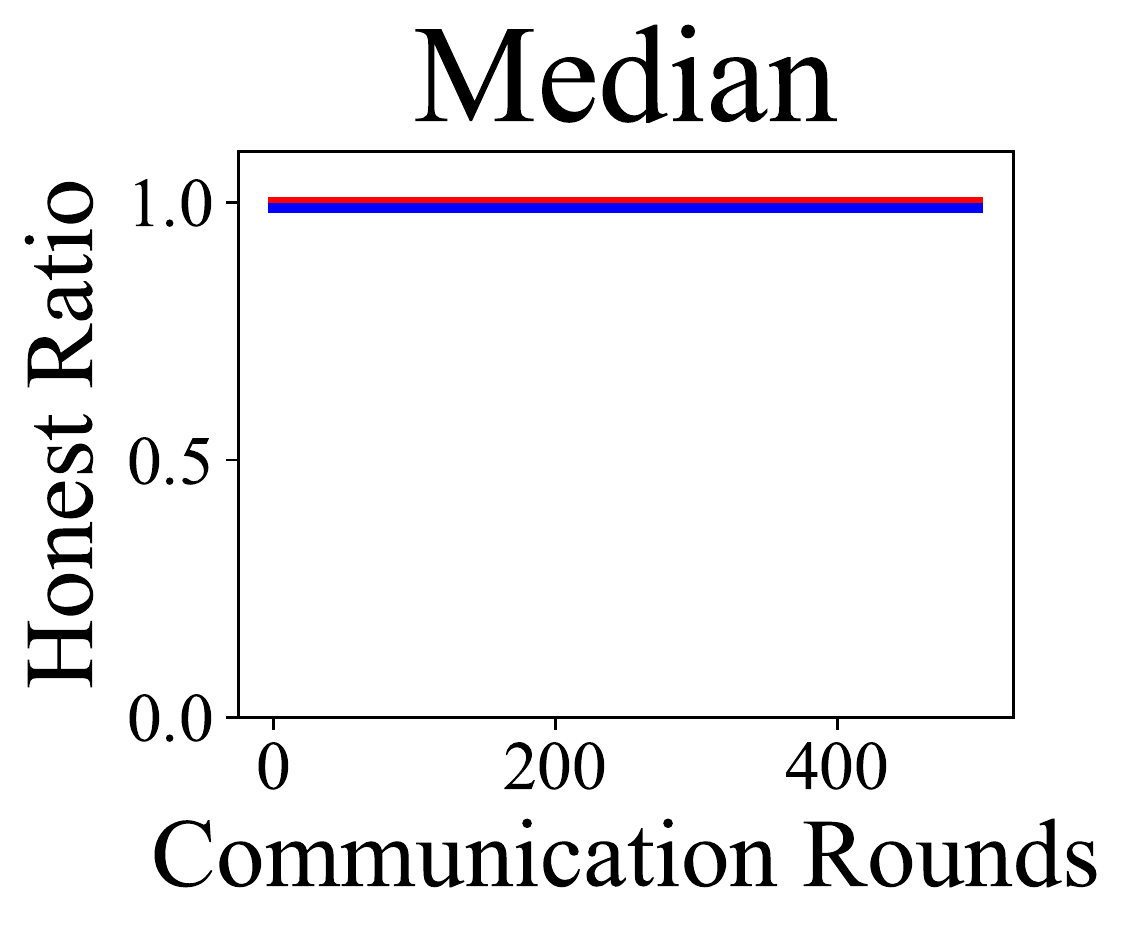}
\includegraphics[width=0.31\textwidth]{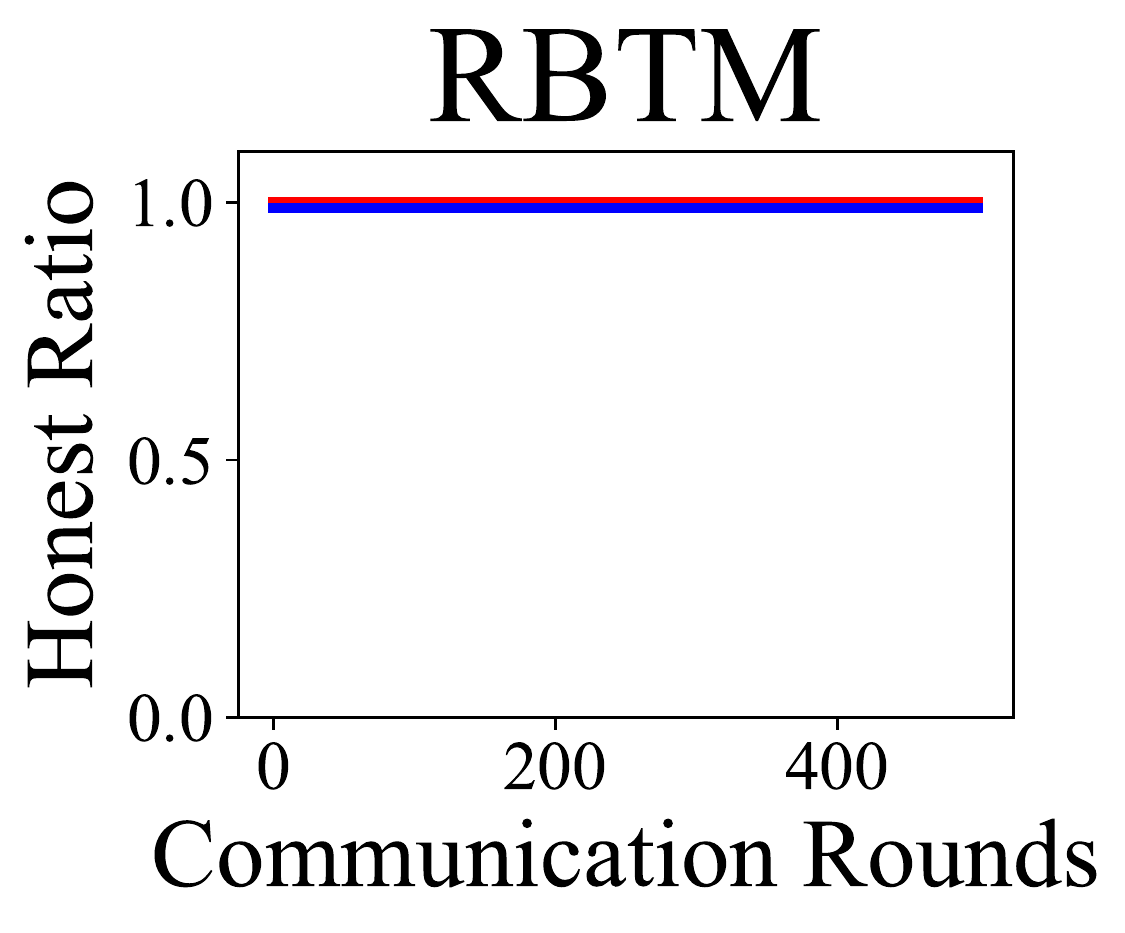}
\includegraphics[width=0.31\textwidth]{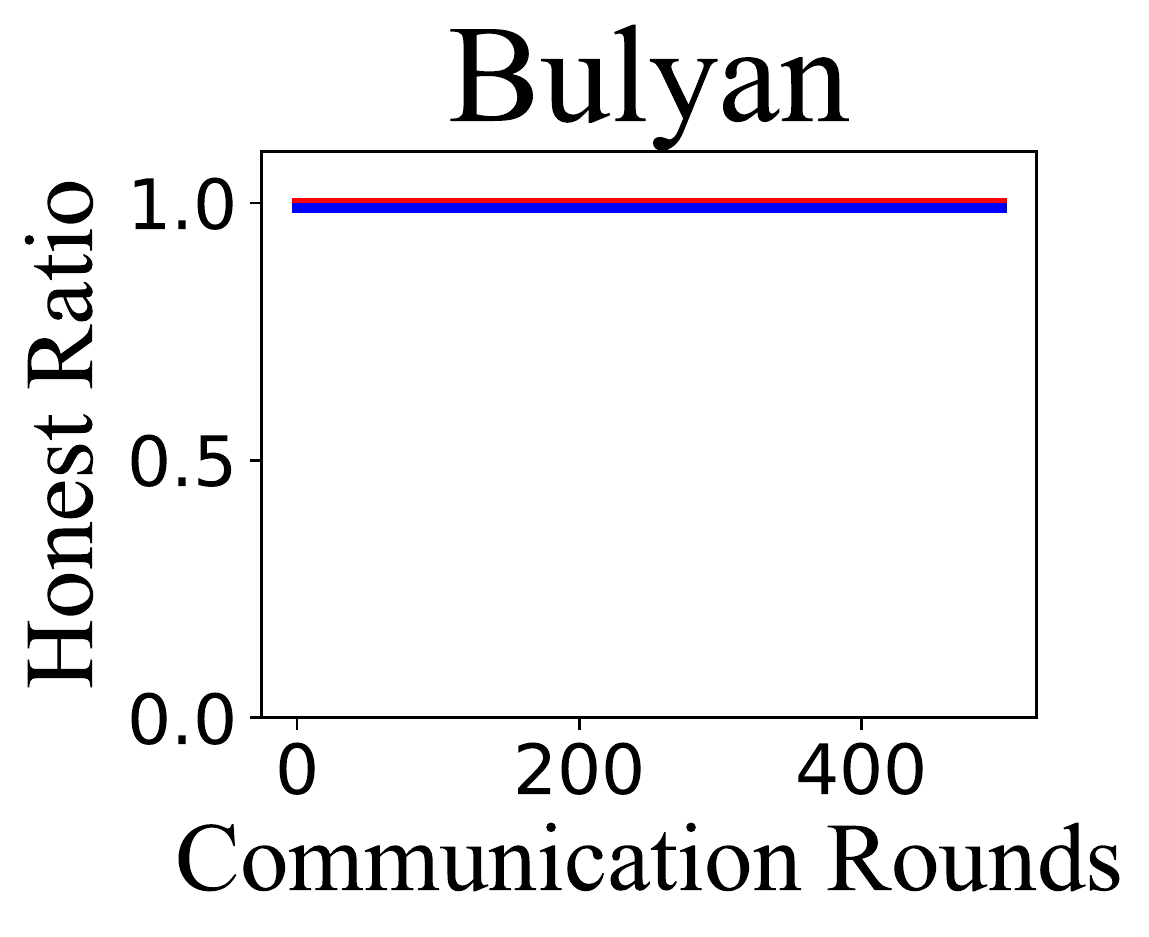}
\caption{The ratio of honest gradients that participate in the aggregation of robust AGRs that aggregate fewer gradients (Median \cite{yin2018mediantrmean}, RBTM \cite{el2021mda}, Bulyan \cite{guerraoui2018bulyan}).}
\label{fig:few_selection}
\end{subfigure}
\hfill
\begin{subfigure}[t]{0.48\textwidth}
\centering
\includegraphics[width=0.31\textwidth]{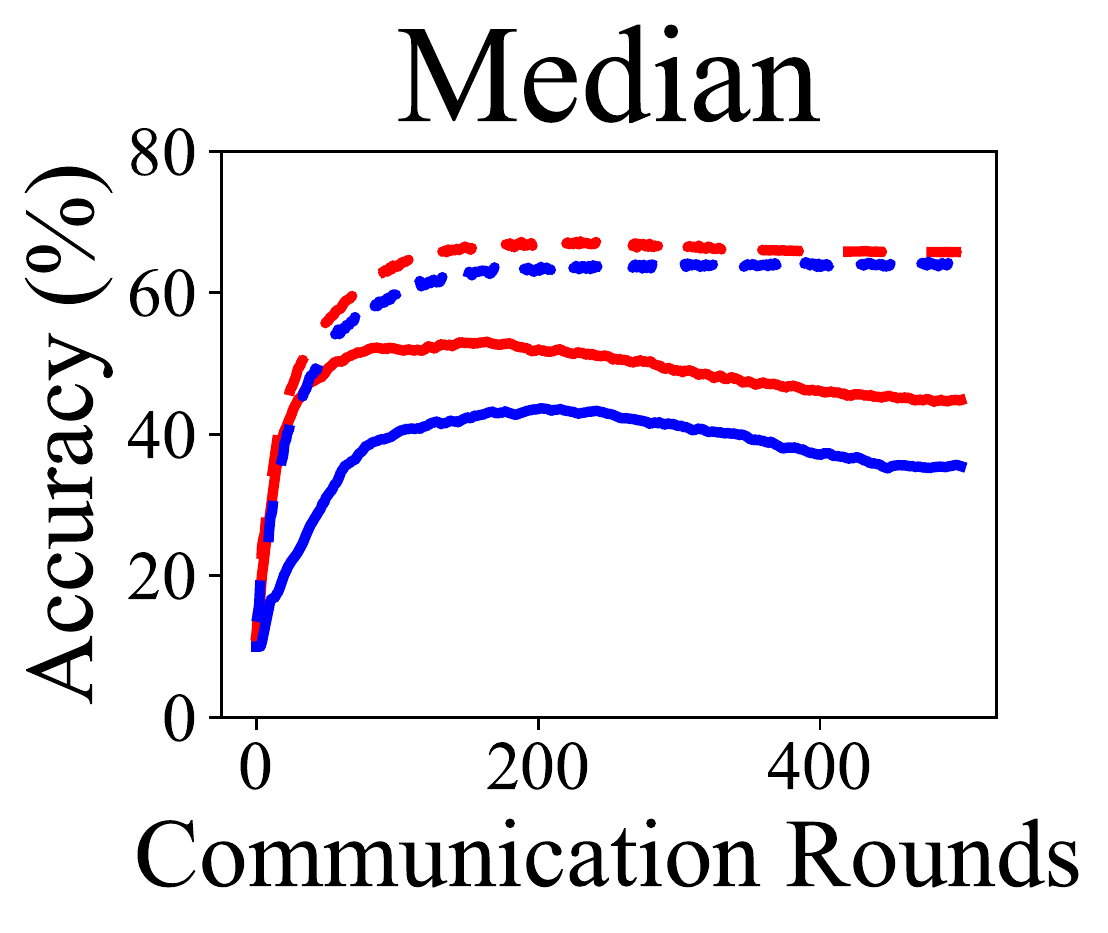}
\includegraphics[width=0.31\textwidth]{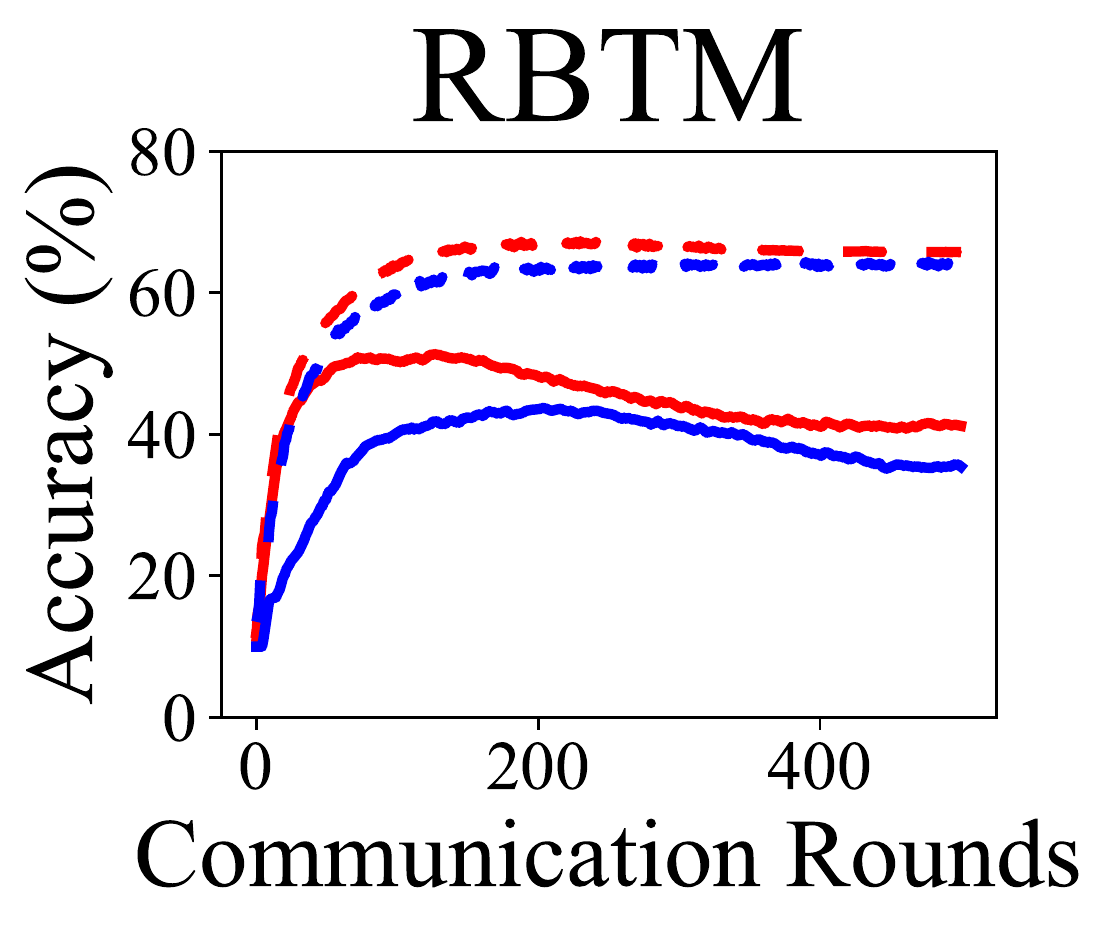}
\includegraphics[width=0.31\textwidth]{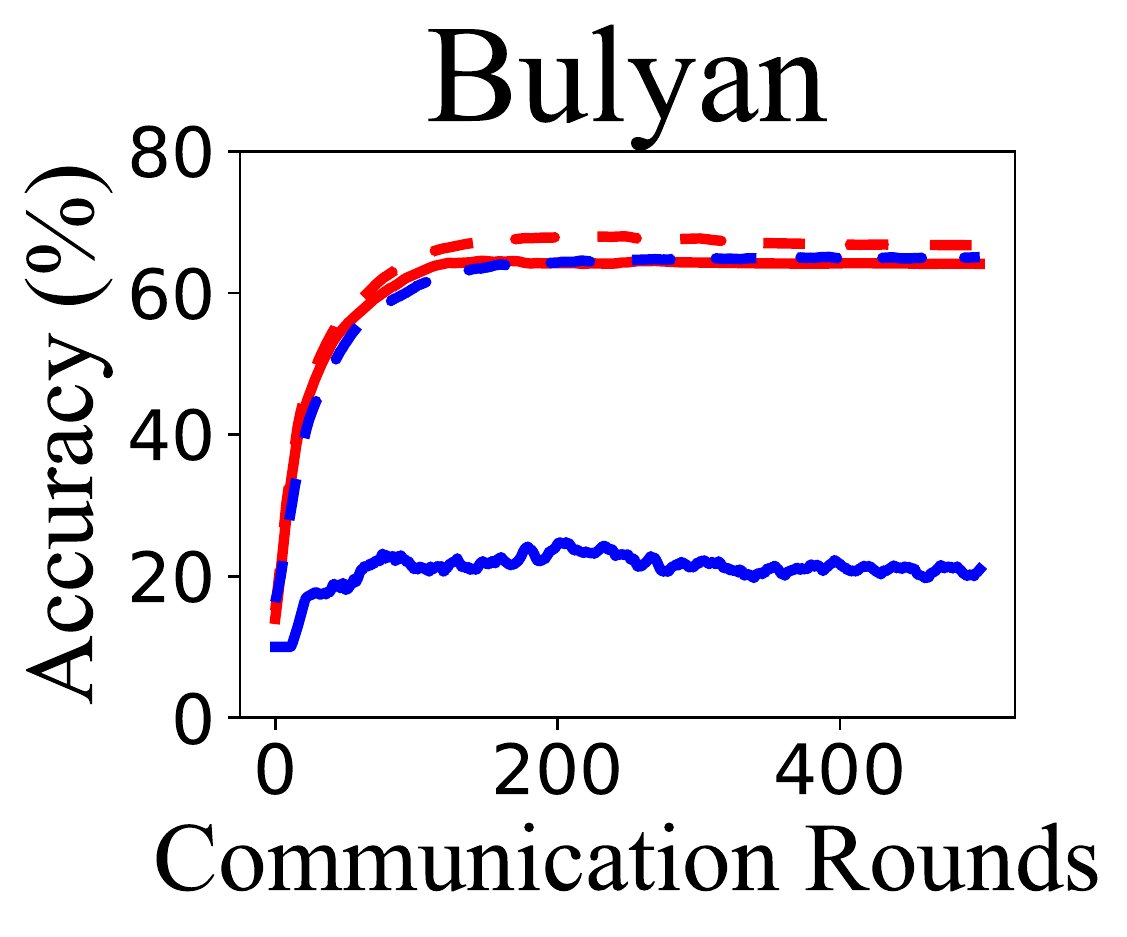}
\caption{The test accuracy of robust AGRs that aggregate few gradients (Median \cite{yin2018mediantrmean}, RBTM \cite{el2021mda}, Bulyan \cite{guerraoui2018bulyan}).}
\label{fig:few_performance}
\end{subfigure}
\caption{The experiments are conducted under the attack of $20\%$ Byzantine clients on CIFAR-10 \cite{krizhevsky2009cifar} dataset in both IID and non-IID settings. More detailed setups are covered in \cref{appsec:motivation_setup}.}
\label{fig:observation}
\end{figure*}

\section{Notations and Preliminaries}
\label{sec:setup}

\textbf{Notations.}
For any positive integer $n\in\sN^+$, we denote the set $\{1,\ldots,n\}$ by $[n]$.
The cardinality of a set $\mathcal{S}$ is denoted by $|\mathcal{S}|$.
We denote the $\ell_2$ norm of vector $\vx$ by $\|\vx\|$.
We use $\component{\vx}{j}$ to represent the $j$-th component of vector $\vx$.
The sub-vector of vector $\vx$ indexed by index set $\indexset{}$ is denoted by $\component{\vx}{\indexset{}}=(\component{\vx}{j_1},\ldots,\component{\vx}{j_k})$, where $\indexset{}=\{j_1,\ldots,j_k\}$, and $k=|\mathcal{J}|$ is the number of indices.
For a random variable $X$, we use $\E[X]$ and $\Var[X]$ to denote the expectation and variance of $X$, respectively.

\textbf{Federated learning.}
We consider the Federated Learning (FL) system with a central server and $\nclients$ clients following \cite{blanchard2017krum, yin2018mediantrmean, chen2022practical}.
Then the objective is to minimize loss $\loss(\param)$ defined as follows.
\begin{gather}
\label{eq:fl_objective}
\loss(\param)=\frac{1}{\nclients}\sum_{i=1}^\nclients\loss[i](\param),\\
\quad\text{where }\loss[i](\param)=\E_{\samples{i}}[\loss(\param;\samples{i})],
i\in[\nclients],
\end{gather}
where $\param$ is the model parameter,
$\loss[i]$ is the loss function on the $i$-th client,
$\samples{i}$ is the data distribution on the $i$-th client,
and $\loss(\param;\sample)$ is the loss function.

In the $t$-th communication round, the server distributes the parameter $\param[t]$ to the clients.
Each client $i$ conducts several epochs of local training on local data to obtain the updated local parameter $\localparam{i}{t}$.
Then, client $i$ computes the local gradient $\sgrad{i}{t}$ as follows and sends it to the server.
\begin{align}
\sgrad{i}{t}=\param[t]-\localparam{i}{t}.
\end{align}

Finally, the server collects the local gradients and uses the average gradient to update the global model.
\begin{align}
\label{eq:fedavg}
    \param[t+1]=\param[t]-\avgsgrad{t},
    \quad\avgsgrad{t}=\frac{1}{n}\sum_{i=1}^n\sgrad{i}{t}.
\end{align}

The process is repeated until the number of communication rounds reaches the set value $T$.

\textbf{Byzantine threat model.}
In real-world applications, not all clients in FL systems are honest. 
In other words, there may exist Byzantine clients in FL systems \cite{blanchard2017krum}.
Suppose that among total $\nclients$ clients, $f$ clients are Byzantine.
Let $\byzantineclients\subseteq[\nclients]$ denote set of Byzantine clients and $\honestclients=[\nclients]\setminus\byzantineclients$ denote the set of honest clients.
In the presence of Byzantine clients, the uploaded message of client $i$ in the $t$-th communication round is
\begin{align}
\label{eq:local_update}
   \sgrad{i}{t}=
   \begin{cases}
      \param[t]-\localparam{i}{t+1}, & i\in\honestclients, \\
      *, & i\in\byzantineclients, 
   \end{cases}
\end{align}
where $*$ represents an arbitrary value.

\textbf{Robust AGRs.}
Most existing Byzantine defenses replace the averaging step with a robust AGR to defend against Byzantine attacks.
More specifically, the server aggregates the gradients and updates the global model as follows.
\begin{align}
\label{eq:robust_aggregation}
\param[t+1]=\param[t]-\aggsgrad{t},
\quad\aggsgrad{t}=\agg(\sgrad{1}{t}, \ldots,\sgrad{\nclients}{t}), 
\end{align}

where $\aggsgrad{t}$ is the aggregated gradient,
and $\agg$ is a robust AGR, e.g., Multi-Krum \cite{blanchard2017krum} and Bulyan \cite{guerraoui2018bulyan}.

For notation simplicity, we omit the superscript $t$ of the gradient symbols when there is no ambiguity in the rest of this paper.

\section{The Challenges of Byzantine Robustness in Non-IID Setting}
\label{sec:motivation}

Most robust AGRs focus on Byzantine robustness in the IID setting \citep{blanchard2017krum,guerraoui2018bulyan}.
When the data is non-IID \cite{kairouz2021advances,zhang2022dense}, the performance of these robust AGRs drops drastically \citep{shejwalkar2021dnc,karimireddy2022bucketing}.
In order to understand the root cause of this performance drop, we perform an experimental study on various robust AGRs.
Particularly, we examine their behaviors under the attack of 20\% Byzantines in both IID and non-IID settings on CIFAR-10 \citep{krizhevsky2009cifar} in \cref{fig:observation}.
More detailed setups are covered in \cref{appsec:motivation_setup}.

Some robust AGRs try to include \emph{all} honest gradients in aggregation (the number of aggregated gradients is no less than $n-f$, i.e., the number of honest clients) \cite{blanchard2017krum,shejwalkar2021dnc,pillutla2019geometric}.
However, they fail to address \emph{the curse of dimensionality} \cite{guerraoui2018bulyan} on heterogeneous data.
Byzantine clients can take advantage of the high dimension of gradients and easily circumvent these defenses. 
As shown in \cref{fig:more_selection}, these defenses include significantly more Byzantine gradients in aggregation in the non-IID setting.
As a result, the global gradient is manipulated away from the optimal gradient,
which leads to an ineffectual global model in the non-IID setting as shown in \cref{fig:more_performance}.

Other robust AGRs aggregate fewer gradients (less than $n-f$) to get rid of Byzantine clients \cite{guerraoui2018bulyan, yin2018mediantrmean, el2021mda}.
The results in \cref{fig:few_selection} imply that they can exclude Byzantine clients from the aggregation in both IID and non-IID settings.
However, their performance still degrades in the non-IID setting as shown in \cref{fig:few_performance}.
In fact, this degradation comes from \emph{gradient heterogeneity} \cite{li2020fedprox,karimireddy2020scaffold} in the non-IID setting.
As a price for removing Byzantine gradients, these robust AGRs exclude a proportion of honest gradients from the aggregation.
Since the honest gradients are heterogeneous, such exclusion causes the aggregated gradient to deviate far from the optimal gradient, i.e., the average of honest gradients.
The deviation further leads to an ineffectual global model.
Therefore, they fail to achieve satisfactory performance in the non-IID setting.

In summary, no existing robust AGR is capable of handling both the curse of dimensionality and gradient heterogeneity at the same time.
A new strategy is needed to tackle both challenges in the non-IID setting. 

\section{Gradient Splitting Based Approach}
\label{sec:proposed}

Our observations in \cref{sec:motivation} clearly motivate the need for a more robust defense to tackle both the curse of dimensionality and gradient heterogeneity to defeat Byzantine attacks in the non-IID setting.
Inspired by these observations, we propose a novel \shorten based approach called \proposedmethod, which consists of three steps as follows.

\textbf{Splitting.}
First, \proposedmethod splits the gradients to mitigate the curse of dimensionality for the next identification step.
The splitting is specified by a partition of set $[\nparams]$, where $\nparams$ is the dimension of gradients.
In particular, we randomly partition $[\nparams]$ into $\ngroups$ subsets, with each subset having no more than $\lceil\nparams/\ngroups\rceil$ dimensions.
Let $\{\indexset{1},\ldots\indexset{\ngroups}\}$ denote the partition.
Each gradient $\sgrad{i}{}$ is correspondingly split into $\ngroups$ sub-vectors as follows.
\begin{align}
    \subsgrad{i}{q}=\component{\sgrad{i}{}}{\indexset{q}},
    \quad i\in[\nclients],q\in[\ngroups],
\end{align}
where $\subsgrad{i}{q}$ is the $q$-th sub-vector of gradient $\sgrad{i}{}$.

\textbf{Identification.}
Then, \proposedmethod applies robust AGR $\agg$ to each group of sub-vectors corresponding to $\indexset{q}$:
\begin{align}
\subaggsgrad{q}=\agg(\subsgrad{1}{q},\ldots,\subsgrad{\nclients}{q}),
\quad q\in[\ngroups],
\end{align}
where $\subaggsgrad{q}$ is the aggregation result of group $q$.
By performing aggregation on each group of low-dimensional sub-vectors separately, \proposedmethod can circumvent the curse of dimensionality and get rid of Byzantine gradients.

Note that $\subaggsgrad{q}$ may still deviate from the optimal gradient due to the gradient heterogeneity \cite{karimireddy2022bucketing} as illustrated in \cref{sec:motivation}.
Therefore, it is \emph{inappropriate} to directly use the aggregation results $\{\subaggsgrad{q},q\in[\ngroups]\}$ as the final output.
Instead, we use $\subaggsgrad{q}$ as an honest reference to compute identification scores for each client as follows.
\begin{align}
\label{eq:group_score}
\groupscore{i}{q}=\|\subsgrad{i}{q}-\subaggsgrad{q}\|,
\quad i\in[\nclients], q\in[\ngroups].
\end{align}
Since the group-wise aggregation result $\subaggsgrad{q}$ can get rid of Byzantine gradients, the identification score $\groupscore{i}{q}$ can provably characterize the potential for the $\subsgrad{i}{q}$ being a sub-vector of a Byzantine gradient.

Then, \proposedmethod collects the identification scores from all groups and computes the final aggregation result.
In particular, the final identification score $\clientscore{i}$ of each client is composed of its identification scores received from all groups as follows.
\begin{align}
    \label{eq:client_score}
    \clientscore{i} = \sum_{q=1}^{\ngroups}\groupscore{i}{q},
    \quad i\in[\nclients].
\end{align}

\textbf{Aggregation.}
To handle the gradient heterogeneity issue, \proposedmethod selects total $\nclients-\nattackers$ gradients with the lowest identification scores for aggregation.
Let $\mathcal{I}$ denote the index set of selected gradients, where $\lvert\gI\rvert=n-f$.
Then the average of selected gradients is output as the final aggregation result as follows:
\begin{align}
    \aggsgrad{} = \frac{1}{\nclients-\nattackers}\sum_{i\in\mathcal{I}}\sgrad{i}{}.
\end{align}

Note that in the second step (Identification) of \proposedmethod, $\agg$ could be any $(f,\lambda)$-resilient AGR (\cref{asp:byzantine_resilience}).
The key difference lies in that all the existing robust AGRs (Multi-Krum, Bulyan, etc.) directly operate on the original gradients; instead, we propose to apply robust AGRs on the split gradients, followed by identification before aggregation.
In this way, we can help enhance the ability to handle both the curse of dimensionality and gradient heterogeneity of the current robust AGRs that satisfy the $(f,\lambda)$-resilient property (\cref{asp:byzantine_resilience}) in the non-IID setting. 
We also analyze the computation cost of our proposed \proposedmethod in \cref{appsec:computation_cost}.

Moreover, our \proposedmethod is a compatible approach that can be combined with most existing robust AGRs, e.g., Multi-Krum \citep{blanchard2017krum}, Bulyan \citep{guerraoui2018bulyan}.

\section{Theoretical Analysis}
\label{sec:theoretical}
In this section, we provide a convergence analysis for our \proposedmethod approach.

We analyze a popular FL model widely considered by
\citet{karimireddy2021history,karimireddy2022bucketing,acharya2022cgm}.
In particular, each local gradient is computed by SGD as follows.
\begin{align}
\label{eq:honest_local_update}
  \sgrad{i}{t}=\eta\nabla\loss(\param[t];\batch{i}{t}),
  \quad i\in\honestclients,
\end{align}
where
$\eta$ is learning rate,
$\batch{i}{t}$ represents a minibatch uniformly sampled from the local data distribution $\samples{i}$ in the $t$-th communication round,
and $\nabla\loss(\param[t],\batch{i}{t})$ represents the gradient of loss over the minibatch $\batch{i}{t}$.

We make the following assumptions, which are standard in FL \cite{karimireddy2021history,karimireddy2022bucketing,acharya2022cgm}.

\begin{restatable}[Unbiased Estimator]{assumption}{aspunbias}
\label{asp:unbias}
The stochastic gradients sampled from any local data distribution are unbiased estimators of local gradients over $\sR^{\nparams}$ for all honest clients, i.e.,
\begin{equation}
\label{eq:unbias}
\begin{gathered}
\E_{\batch{i}{t}}[\nabla\loss[i](\param;\batch{i}{t})]=\nabla\loss[i](\param),\\
\forall\param\in\sR^{\nparams}, i\in\honestclients, t\in\sN^+.
\end{gathered}
\end{equation}
\end{restatable}

\begin{restatable}[Bounded Variance]{assumption}{aspvar}
\label{asp:bounded_var}
The variance of stochastic gradients sampled from any local data distribution is uniformly bounded over $\sR^{\nparams}$ for all honest clients, i.e., there exists $\sigma\ge0$ such that
\begin{equation}
\label{eq:bounded_var}
\begin{gathered}
\E\|\nabla\loss[i](\param;\batch{i}{t})-\nabla\loss[i](\param)\|^2\le\sigma^2,\\
\forall\param\in\sR^{\nparams}, i\in\honestclients, t\in\sN^+.
\end{gathered}
\end{equation}
\end{restatable}

\begin{restatable}[Gradient Dissimilarity]{assumption}{aspdissim}
\label{asp:dissimilarity}
The difference between the local gradients and the global gradient is uniformly bounded over $\sR^{\nparams}$ for all honest clients, i.e., there exists $\kappa\ge0$ such that
\begin{align}
\label{eq:dissimilarity}
\|\nabla\loss[i](\param)-\nabla\loss(\param)\|^2 \le\kappa^2,
\quad\forall\param\in\sR^{\nparams}, i\in\honestclients.
\end{align}
\end{restatable}

We consider arbitrary non-convex loss function $\loss(\cdot)$ that satisfies the following Lipschitz condition. This condition is widely applied in the convergence analysis of Byzantine-robust federated learning \cite{karimireddy2022bucketing,allen2020safeguard,el2021mda}.
\begin{restatable}[Lipschitz Smoothness]{assumption}{asplipschitz}
\label{asp:lipschitz}
The loss function is $L$-Lipschitz smooth over $\sR^d$, i.e., 
\begin{equation}
\begin{gathered}
\|\nabla\loss(\vw)-\nabla\loss(\vw')\|\le L\|\vw-\vw'\|,\\
\forall\vw,\vw'\in\sR^d.
\end{gathered}
\end{equation}
\end{restatable}

We consider robust AGRs that satisfy the following robustness criterion (\cref{asp:byzantine_resilience}) introduced by \citet{farhadkhani2022reasm}.
A wide class of state-of-the-art robust AGRs satisfy this criterion \citep{farhadkhani2022reasm}.
\begin{restatable}[$(f,\lambda)$-resilient]{definition}{defbyzresil}
\label{asp:byzantine_resilience}
For integer $f<n/2$ and real value $\lambda>0$, an AGR $\agg$ is called $(f,\lambda)$-resilient if for any input $\{\vx_1,\ldots,\vx_\nclients\}$ and any set $\gS\subseteq[\nclients]$ of size $\nclients-f$,
the output of $\agg$ satisfies:
\begin{equation}
\begin{gathered}
\|\agg(\vx_1, \ldots,\vx_n) - \Bar{\vx}_\gS\|\le \lambda\max_{i,i'\in\gS}\|\vx_i-\vx_{i'}\|, \\
\end{gathered}
\end{equation}
where $\Bar{\vx}_\gS=\sum_{i\in\gS}\vx_i/|\gS|$.
\end{restatable}

We show that given any $(f,\lambda)$-resilient base AGR $\agg$, our \proposedmethod can help the global model to reach a better parameter.

\begin{restatable}[]{proposition}{propconvergence}
\label{prop:convergence}
Suppose \cref{asp:unbias,asp:bounded_var,asp:dissimilarity,asp:lipschitz} hold, and let learning rate $\eta=1/2L$.
Given any $(f,\lambda)$-resilient robust AGR $\agg$, we start from $\vw^0$ and run \proposedmethod for $T$ communication rounds, it satisfies
\begin{align}
\label{eq:prop_convergence}
\loss(\vw^0)\ge\frac{3}{16L}\sum_{t=1}^T(\|\nabla\loss(\vw^t)\|^2-e^2),
\end{align}
where
\begin{align}
\label{eq:prop_error}
e^2
={}&\gO((\kappa^2+\sigma^2)\\
&\cdot(1+\frac{n-f+1}{p})(1+\lambda^2+\frac{1}{n-f})\frac{f^2}{(n-f)^2}).
\end{align}
\end{restatable}

Please refer to \cref{appsec:proof} for the proof.
\cref{prop:convergence} provides an upper bound for the sum of gradient norms in the presence of Byzantine gradients.
\cref{eq:prop_convergence} indicates that as the number of communication rounds increases, we can find an approximate optimal parameter $\bm{w}$ such that $\|\nabla\mathcal{L}(\bm{w})\|$ can be arbitrary close to $e$.
$\kappa^2$ and $\sigma^2$ in \cref{eq:prop_error} are positively related to the gradient dimension $d$ \citep{guerraoui2018bulyan}. Therefore, the convergence error $e$ grows larger when $d$ increases.
As the number of sub-vectors $\ngroups$ increases, the approximation becomes better, i.e., $e^2$ decreases, which validates the efficacy of our approach.
From another aspect, \cref{prop:convergence} also characterizes the fundamental difficulties of Byzantine-robust federated learning in the non-IID setting.
The negative term $-e^2$ on the RHS of \cref{eq:prop_convergence} implies that FL may never converge to an optimal parameter.
By contrast, the global model may wander among sub-optimal points.
What's more, even after reaching the convergence point, the global model may step into another sub-optimal in the next communication round.
It aligns with the previous lower bound in \cite{karimireddy2022bucketing}.
A detailed comparison of the convergence results between our approach and recent works is presented in \cref{appsec:comparaion}.

\section{Experiments}
\label{sec:experiments}

\begin{table*}[htbp]
\caption{Accuracy
(mean$\pm$std)
of different defenses under 6 attacks on CIFAR-10, CIFAR-100, FEMNIST, and ImageNet-12.}
\label{tbl:main_experiments}
\begin{center}
\scalebox{\scalefactor}{
\begin{tabular}{llcccccc}																											
\toprule																											
Dataset	&	Attack	&	BitFlip	&	LabelFlip	&	LIE	&	Min-Max	&	Min-Sum	&	IPM	\\												
\midrule																											
\multirow{12}*{CIFAR-10}	&	Multi-Krum	&	43.19	$\pm$	0.38	&	43.90	$\pm$	0.03	&	37.03	$\pm$	1.62	&	39.06	$\pm$	0.07	&	23.68	$\pm$	0.18	&	36.47	$\pm$	0.22	\\
	&	\proposedmethod(Multi-Krum)	&	\textbf{59.23}	$\pm$	0.55	&	\textbf{61.47}	$\pm$	0.26	&	\textbf{55.66}	$\pm$	0.93	&	\textbf{49.19}	$\pm$	0.72	&	\textbf{53.59}	$\pm$	0.96	&	\textbf{56.94}	$\pm$	3.60	\\
		\cmidrule{2-8}																									
	&	Bulyan	&	54.10	$\pm$	0.19	&	55.12	$\pm$	0.14	&	30.58	$\pm$	0.75	&	29.03	$\pm$	1.10	&	46.19	$\pm$	0.92	&	33.88	$\pm$	0.61	\\
	&	\proposedmethod(Bulyan)	&	\textbf{59.14}	$\pm$	0.01	&	\textbf{61.21}	$\pm$	0.60	&	\textbf{48.90}	$\pm$	0.83	&	\textbf{48.35}	$\pm$	1.58	&	\textbf{53.74}	$\pm$	0.71	&	\textbf{56.53}	$\pm$	1.51	\\
		\cmidrule{2-8}																									
	&	Median	&	45.41	$\pm$	0.44	&	51.88	$\pm$	0.62	&	28.75	$\pm$	0.35	&	32.72	$\pm$	0.81	&	37.39	$\pm$	0.90	&	43.21	$\pm$	0.47	\\
	&	\proposedmethod(Median)	&	\textbf{59.28}	$\pm$	0.24	&	\textbf{61.24}	$\pm$	1.34	&	\textbf{46.60}	$\pm$	0.13	&	\textbf{49.37}	$\pm$	1.13	&	\textbf{53.32}	$\pm$	1.90	&	\textbf{56.33}	$\pm$	0.82	\\
		\cmidrule{2-8}																									
	&	RFA	&	49.61	$\pm$	0.31	&	44.35	$\pm$	0.31	&	15.39	$\pm$	0.37	&	16.62	$\pm$	0.83	&	18.22	$\pm$	0.43	&	45.92	$\pm$	0.13	\\
	&	\proposedmethod(RFA)	&	\textbf{53.35}	$\pm$	0.30	&	\textbf{62.25}	$\pm$	0.56	&	\textbf{52.69}	$\pm$	0.89	&	\textbf{52.64}	$\pm$	1.48	&	\textbf{56.16}	$\pm$	0.91	&	\textbf{62.26}	$\pm$	1.27	\\
		\cmidrule{2-8}																									
	&	DnC	&	58.63	$\pm$	1.29	&	60.82	$\pm$	1.56	&	61.07	$\pm$	0.72	&	60.42	$\pm$	0.59	&	53.71	$\pm$	0.96	&	\textbf{59.99}	$\pm$	0.82	\\
	&	\proposedmethod(DnC)	&	\textbf{58.96}	$\pm$	0.60	&	\textbf{61.02}	$\pm$	0.27	&	\textbf{61.87}	$\pm$	0.51	&	\textbf{61.04}	$\pm$	1.18	&	\textbf{54.36}	$\pm$	1.12	&	57.92	$\pm$	1.71	\\
		\cmidrule{2-8}																									
	&	RBTM	&	54.27	$\pm$	1.63	&	59.60	$\pm$	1.76	&	47.67	$\pm$	2.51	&	49.02	$\pm$	0.31	&	50.74	$\pm$	0.06	&	55.27	$\pm$	1.60	\\
	&	\proposedmethod(RBTM)	&	\textbf{59.41}	$\pm$	0.20	&	\textbf{60.75}	$\pm$	0.19	&	\textbf{52.10}	$\pm$	1.28	&	\textbf{49.60}	$\pm$	0.17	&	\textbf{53.63}	$\pm$	0.58	&	\textbf{56.65}	$\pm$	1.52	\\
\midrule																											
\multirow{12}*{CIFAR-100}	&	Multi-Krum	&	34.27	$\pm$	0.28	&	35.57	$\pm$	0.94	&	17.17	$\pm$	0.08	&	16.77	$\pm$	0.78	&	22.89	$\pm$	0.61	&	15.93	$\pm$	2.00	\\
	&	\proposedmethod(Multi-Krum)	&	\textbf{42.41}	$\pm$	0.58	&	\textbf{42.55}	$\pm$	0.12	&	\textbf{27.81}	$\pm$	0.32	&	\textbf{31.18}	$\pm$	1.48	&	\textbf{41.33}	$\pm$	0.50	&	\textbf{42.62}	$\pm$	1.53	\\
		\cmidrule{2-8}																									
	&	Bulyan	&	35.77	$\pm$	0.18	&	42.60	$\pm$	0.07	&	35.41	$\pm$	0.40	&	35.53	$\pm$	1.38	&	39.13	$\pm$	0.12	&	40.27	$\pm$	1.64	\\
	&	\proposedmethod(Bulyan)	&	\textbf{42.28}	$\pm$	1.61	&	\textbf{43.77}	$\pm$	0.46	&	\textbf{38.39}	$\pm$	0.19	&	\textbf{36.33}	$\pm$	1.51	&	\textbf{40.73}	$\pm$	0.39	&	\textbf{42.88}	$\pm$	0.14	\\
		\cmidrule{2-8}																									
	&	Median	&	36.62	$\pm$	0.12	&	41.64	$\pm$	0.76	&	22.75	$\pm$	0.04	&	23.21	$\pm$	0.71	&	30.68	$\pm$	0.26	&	40.98	$\pm$	0.38	\\
	&	\proposedmethod(Median)	&	\textbf{42.41}	$\pm$	0.66	&	\textbf{42.62}	$\pm$	0.09	&	\textbf{35.16}	$\pm$	1.08	&	\textbf{36.46}	$\pm$	0.10	&	\textbf{41.08}	$\pm$	0.04	&	\textbf{43.63}	$\pm$	2.85	\\
		\cmidrule{2-8}																									
	&	RFA	&	21.32	$\pm$	0.84	&	28.76	$\pm$	1.33	&	25.63	$\pm$	0.20	&	26.46	$\pm$	1.83	&	28.33	$\pm$	0.93	&	21.36	$\pm$	0.54	\\
	&	\proposedmethod(RFA)	&	\textbf{42.64}	$\pm$	0.44	&	\textbf{42.42}	$\pm$	0.25	&	\textbf{26.30}	$\pm$	1.08	&	\textbf{30.30}	$\pm$	0.12	&	\textbf{41.09}	$\pm$	0.66	&	\textbf{43.45}	$\pm$	0.52	\\
		\cmidrule{2-8}																									
	&	DnC	&	41.77	$\pm$	0.62	&	42.93	$\pm$	0.07	&	42.95	$\pm$	1.03	&	40.15	$\pm$	0.70	&	40.02	$\pm$	1.07	&	41.23	$\pm$	2.29	\\
	&	\proposedmethod(DnC)	&	\textbf{43.35}	$\pm$	0.41	&	\textbf{43.57}	$\pm$	1.11	&	\textbf{43.64}	$\pm$	0.11	&	\textbf{41.66}	$\pm$	0.78	&	\textbf{41.02}	$\pm$	1.39	&	\textbf{43.25}	$\pm$	0.43	\\
		\cmidrule{2-8}																									
	&	RBTM	&	36.35	$\pm$	0.17	&	42.67	$\pm$	1.55	&	24.06	$\pm$	0.09	&	26.24	$\pm$	1.04	&	36.51	$\pm$	0.40	&	43.12	$\pm$	1.12	\\
	&	\proposedmethod(RBTM)	&	\textbf{43.44}	$\pm$	0.81	&	\textbf{43.19}	$\pm$	2.65	&	\textbf{33.14}	$\pm$	0.58	&	\textbf{34.35}	$\pm$	0.76	&	\textbf{41.51}	$\pm$	0.93	&	\textbf{43.20}	$\pm$	0.76	\\
\midrule																											
\multirow{12}*{FEMNIST}	&	Multi-Krum	&	67.65	$\pm$	0.23	&	57.43	$\pm$	1.25	&	44.58	$\pm$	0.07	&	28.32	$\pm$	0.31	&	29.98	$\pm$	0.45	&	12.26	$\pm$	1.34	\\
	&	\proposedmethod(Multi-Krum)	&	\textbf{84.29}	$\pm$	1.76	&	\textbf{85.45}	$\pm$	0.40	&	\textbf{74.76}	$\pm$	1.74	&	\textbf{57.46}	$\pm$	0.33	&	\textbf{70.65}	$\pm$	1.35	&	\textbf{81.46}	$\pm$	0.18	\\
		\cmidrule{2-8}																									
	&	Bulyan	&	77.58	$\pm$	1.30	&	79.39	$\pm$	2.14	&	56.43	$\pm$	0.45	&	35.10	$\pm$	0.69	&	44.83	$\pm$	1.40	&	5.91	$\pm$	0.17	\\
	&	\proposedmethod(Bulyan)	&	\textbf{84.90}	$\pm$	0.69	&	\textbf{83.68}	$\pm$	0.76	&	\textbf{71.43}	$\pm$	1.07	&	\textbf{66.22}	$\pm$	0.47	&	\textbf{71.76}	$\pm$	0.99	&	\textbf{82.97}	$\pm$	1.04	\\
		\cmidrule{2-8}																									
	&	Median	&	80.25	$\pm$	0.06	&	76.86	$\pm$	1.96	&	64.88	$\pm$	0.23	&	50.67	$\pm$	0.37	&	61.33	$\pm$	0.13	&	71.98	$\pm$	0.77	\\
	&	\proposedmethod(Median)	&	\textbf{84.59}	$\pm$	0.14	&	\textbf{85.67}	$\pm$	0.48	&	\textbf{76.19}	$\pm$	0.43	&	\textbf{65.84}	$\pm$	0.41	&	\textbf{70.84}	$\pm$	0.86	&	\textbf{82.18}	$\pm$	0.40	\\
		\cmidrule{2-8}																									
	&	RFA	&	5.46	$\pm$	0.06	&	5.46	$\pm$	0.01	&	5.46	$\pm$	0.05	&	5.46	$\pm$	0.03	&	5.46	$\pm$	0.02	&	5.59	$\pm$	0.09	\\
	&	\proposedmethod(RFA)	&	\textbf{84.86}	$\pm$	0.78	&	\textbf{84.59}	$\pm$	0.20	&	\textbf{69.82}	$\pm$	0.33	&	\textbf{69.18}	$\pm$	0.09	&	\textbf{77.67}	$\pm$	1.31	&	\textbf{86.08}	$\pm$	2.51	\\
		\cmidrule{2-8}																									
	&	DnC	&	8.90	$\pm$	0.31	&	77.71	$\pm$	0.03	&	78.52	$\pm$	0.28	&	8.29	$\pm$	0.37	&	74.18	$\pm$	0.03	&	74.70	$\pm$	1.57	\\
	&	\proposedmethod(DnC)	&	\textbf{84.71}	$\pm$	0.39	&	\textbf{85.39}	$\pm$	0.64	&	\textbf{82.54}	$\pm$	0.26	&	\textbf{74.37}	$\pm$	0.50	&	\textbf{75.41}	$\pm$	0.22	&	\textbf{82.73}	$\pm$	1.22	\\
		\cmidrule{2-8}																									
	&	RBTM	&	82.57	$\pm$	0.34	&	81.57	$\pm$	1.12	&	59.93	$\pm$	0.20	&	65.20	$\pm$	0.60	&	71.82	$\pm$	0.73	&	76.88	$\pm$	1.75	\\
	&	\proposedmethod(RBTM)	&	\textbf{84.89}	$\pm$	1.94	&	\textbf{85.44}	$\pm$	0.20	&	\textbf{73.38}	$\pm$	0.31	&	\textbf{66.24}	$\pm$	0.94	&	\textbf{75.50}	$\pm$	1.13	&	\textbf{82.58}	$\pm$	1.85	\\
\midrule																											
\multirow{12}*{ImageNet-12}	&	Multi-Krum	&	44.36	$\pm$	1.52	&	34.04	$\pm$	1.69	&	45.38	$\pm$	1.04	&	48.72	$\pm$	0.16	&	57.69	$\pm$	0.30	&	33.14	$\pm$	0.86	\\
	&	\proposedmethod(Multi-Krum)	&	\textbf{66.79}	$\pm$	1.08	&	\textbf{63.04}	$\pm$	0.14	&	\textbf{57.15}	$\pm$	0.19	&	\textbf{59.94}	$\pm$	0.32	&	\textbf{64.07}	$\pm$	1.38	&	\textbf{61.92}	$\pm$	0.04	\\
		\cmidrule{2-8}																									
	&	Bulyan	&	62.28	$\pm$	0.84	&	59.84	$\pm$	1.09	&	48.04	$\pm$	2.22	&	48.97	$\pm$	1.87	&	59.94	$\pm$	0.51	&	60.67	$\pm$	0.07	\\
	&	\proposedmethod(Bulyan)	&	\textbf{66.76}	$\pm$	0.72	&	\textbf{62.28}	$\pm$	0.32	&	\textbf{57.44}	$\pm$	0.39	&	\textbf{58.81}	$\pm$	0.05	&	\textbf{65.00}	$\pm$	0.08	&	\textbf{62.76}	$\pm$	0.14	\\
		\cmidrule{2-8}																									
	&	Median	&	55.93	$\pm$	0.55	&	58.14	$\pm$	0.18	&	46.67	$\pm$	1.01	&	49.07	$\pm$	1.19	&	58.40	$\pm$	0.03	&	43.62	$\pm$	1.72	\\
	&	\proposedmethod(Median)	&	\textbf{66.28}	$\pm$	0.41	&	\textbf{62.34}	$\pm$	1.10	&	\textbf{60.74}	$\pm$	1.24	&	\textbf{59.26}	$\pm$	0.31	&	\textbf{64.78}	$\pm$	2.10	&	\textbf{62.24}	$\pm$	0.51	\\
		\cmidrule{2-8}																									
	&	RFA	&	61.12	$\pm$	1.26	&	61.31	$\pm$	1.68	&	49.49	$\pm$	1.33	&	53.04	$\pm$	0.13	&	61.92	$\pm$	0.67	&	63.97	$\pm$	0.93	\\
	&	\proposedmethod(RFA)	&	\textbf{66.92}	$\pm$	1.58	&	\textbf{63.88}	$\pm$	0.94	&	\textbf{61.41}	$\pm$	0.02	&	\textbf{59.42}	$\pm$	0.64	&	\textbf{67.02}	$\pm$	0.54	&	\textbf{66.67}	$\pm$	0.38	\\
		\cmidrule{2-8}																									
	&	DnC	&	54.94	$\pm$	0.04	&	5.59	$\pm$	0.06	&	58.01	$\pm$	1.52	&	58.11	$\pm$	0.41	&	60.42	$\pm$	1.60	&	59.99	$\pm$	0.50	\\
	&	\proposedmethod(DnC)	&	\textbf{65.19}	$\pm$	1.63	&	\textbf{63.01}	$\pm$	0.27	&	\textbf{64.42}	$\pm$	0.19	&	\textbf{65.03}	$\pm$	1.23	&	\textbf{65.38}	$\pm$	1.68	&	\textbf{65.03}	$\pm$	0.04	\\
		\cmidrule{2-8}																									
	&	RBTM	&	60.06	$\pm$	1.76	&	60.44	$\pm$	0.37	&	55.77	$\pm$	0.82	&	57.50	$\pm$	0.10	&	63.91	$\pm$	0.78	&	56.19	$\pm$	1.05	\\
	&	\proposedmethod(RBTM)	&	\textbf{66.99}	$\pm$	0.38	&	\textbf{61.92}	$\pm$	1.22	&	\textbf{59.87}	$\pm$	0.72	&	\textbf{59.81}	$\pm$	1.34	&	\textbf{64.94}	$\pm$	0.72	&	\textbf{63.40}	$\pm$	0.97	\\
\bottomrule																											
\end{tabular}
}
\end{center}
\end{table*}

\begin{table*}[t]
    \centering
    \caption{Accuracy of different robust AGRs combined with Bucketing or \proposedmethod under six attacks on CIFAR-10.}
    \label{tbl:bucketing}
    \scalebox{\scalefactor}{
    \begin{tabular}{lcccccc}
\toprule
Attack & BitFlip & LabelFlip & LIE & Min-Max & Min-Sum & IPM \\
\midrule
Bucketing (Multi-Krum) & 47.87 & 49.86 & 45.90 & 43.53 & 44.92 & 50.28 \\
\proposedmethod (Multi-Krum) & \textbf{59.23} & \textbf{61.47} & \textbf{55.66} & \textbf{49.19} & \textbf{53.59} & \textbf{56.94} \\
\midrule
Bucketing (Bulyan) & 51.79 & 61.16 & 46.02 & 45.90 & 52.30 & 56.44 \\
\proposedmethod (Bulyan) & \textbf{59.14} & \textbf{61.21} & \textbf{48.90} & \textbf{48.35} & \textbf{53.74} & \textbf{56.53} \\
\midrule
Bucketing (Median) & 53.17 & 59.50 & \textbf{47.13} & 47.93 & 51.52 & 52.69 \\
\proposedmethod (Median) & \textbf{59.28} & \textbf{61.24} & 46.60 & \textbf{49.37} & \textbf{53.32} & \textbf{56.33} \\
\midrule
Bucketing (RFA) & 52.55 & 58.44 & 48.71 & 47.51 & 52.29 & 55.19 \\
\proposedmethod (RFA) & \textbf{53.35} & \textbf{62.25} & \textbf{52.69} & \textbf{52.64} & \textbf{56.16} & \textbf{62.26} \\
\midrule
Bucketing (DnC) & 57.79 & 59.39 & 57.53 & 55.09 & 53.83 & 54.01 \\
\proposedmethod (DnC) & \textbf{58.96} & \textbf{61.02} & \textbf{61.87} & \textbf{61.04} & \textbf{54.36} & \textbf{57.92} \\
\midrule
Bucketing (RBTM) & 53.25 & 60.10 & 51.87 & 49.32 & 53.56 & 53.77 \\
\proposedmethod (RBTM) & \textbf{59.41} & \textbf{60.75} & \textbf{52.10} & \textbf{49.60} & \textbf{53.63} & \textbf{56.65} \\
\bottomrule
\end{tabular}

    }
\end{table*}

\begin{table*}[t]
\begin{center}
\caption{Accuracy of \proposedmethod with different number of sub-vectors $p$ under LIE attack on CIFAR-10. $d$ represents the number of model parameters.}
\label{tbl:ps}
\scalebox{\scalefactor}{
\begin{tabular}{lcccccc}
\toprule
$p$ & 100 & 1000 & 10000 & 100000 & 1000000 & 2472266 ($d$) \\
\midrule
\proposedmethod (Multi-Krum) & 55.07 & 63.23 & \textbf{63.86} & 60.16 & 58.31 & 57.70 \\
GAS (Bulyan) & 50.29 & 57.11 & 59.82 & 60.42 & \textbf{60.47} & 59.90 \\
\bottomrule
\end{tabular}

}
\end{center}
\end{table*}

\subsection{Experimental Setups}
\label{subsec:experimental_setups}

\textbf{Datasets.} Our experiments are conducted on four real-world datasets: CIFAR-10 \cite{krizhevsky2009cifar}, CIFAR-100 \cite{krizhevsky2009cifar}, a subset of ImageNet \cite{russakovsky2015imagenet} refered as ImageNet-12 \cite{li2021imagenet12} and FEMNIST \cite{caldas2018leaf}.

\textbf{Data distribution.}
For CIFAR-10, CIFAR-100, and ImageNet-12, we use Dirichlet distribution to generate non-IID data by following \citet{yurochkin2019bayesian, li2021federated}.
We follow \citet{li2021federated} and set the number of clients $\nclients=50$ and the concentration parameter of Dirichlet distribution $\beta=0.5$ as default.
FEMNIST is a dataset with a natural non-IID partition.
In particular, the data is partitioned into 3,597 clients based on the writer of the digit/character.
For each client, we randomly sample 0.9 portion of data as training data and let the rest 0.1 portion of data be test data by following \citet{caldas2018leaf}.

\begin{table*}[t]
\begin{center}
\caption{The accuracy of GAS with $\delta=0.1, 0.3$ under 20\% LIE attack on CIFAR-10. N/A represents the case where the number of Byzantine clients $f$ is known to the server and the server can exclude exactly $f$ clients, i.e., $\delta$ is N/A.
}
\label{tbl:unknown_f}
\begin{tabular}{lcccccc}
\toprule
$\delta$ & Multi-Krum & Bulyan & median & RFA & DnC & RBTM \\
\midrule
N/A & \textbf{55.66} & \textbf{48.90} & \textbf{46.60} & \textbf{52.69} & \textbf{61.87} & \textbf{52.10} \\
0.1 & 54.30 & 46.94 & 45.74 & 52.21 & 56.93 & 51.47 \\
0.3 & 50.48 & 44.96 & 43.09 & 52.04 & 60.50 & 50.21 \\
\bottomrule
\end{tabular}
\end{center}
\end{table*}

\begin{table*}[t]
\caption{
Accuracy (mean$\pm$std) of different defenses against LIE attack under different non-IID levels on CIFAR-10.
A smaller $\beta$ implies a higher non-IID level.
}
\label{tbl:non_iid}
\begin{center}
\scalebox{\scalefactor}{

\begin{tabular}{c|cc|cc|cc}
\toprule
$\beta$ & Multi-Krum & \proposedmethod (Multi-Krum) & Bulyan & \proposedmethod (Bulyan) & Median & \proposedmethod (Median) \\
\midrule
0.3 & 12.19 $\pm$ 1.04 & \textbf{52.80} $\pm$ 0.74 & 28.16 $\pm$ 0.44 & \textbf{42.81} $\pm$ 0.63 & 25.62 $\pm$ 0.83 & \textbf{40.97} $\pm$ 0.89 \\
0.7 & 31.01 $\pm$ 0.54 & \textbf{55.64} $\pm$ 0.60 & 44.72 $\pm$ 1.43 & \textbf{51.29} $\pm$ 0.35 & 34.04 $\pm$ 0.29 & \textbf{53.34} $\pm$ 0.08 \\
\midrule
$\beta$ & RFA & \proposedmethod (RFA) & DnC & \proposedmethod (DnC) & RBTM & \proposedmethod (RBTM) \\
\midrule
0.3 & 20.08 $\pm$ 0.13 & \textbf{48.77} $\pm$ 0.84 & 59.99 $\pm$ 1.81 & \textbf{60.21} $\pm$ 0.62 & 37.67 $\pm$ 0.18 & \textbf{49.27} $\pm$ 0.05 \\
0.7 & 18.11 $\pm$ 0.24 & \textbf{53.25} $\pm$ 1.41 & 62.15 $\pm$ 0.73 & \textbf{62.48} $\pm$ 0.52 & 48.43 $\pm$ 0.22 & \textbf{52.25} $\pm$ 1.16 \\
\bottomrule
\end{tabular}

}
\end{center}
\end{table*}

\begin{table*}[t]
\caption{Accuracy (mean$\pm$std) of different defenses against LIE attack with different Byzantine client numbers $\nattackers=\{5, 15\}$ on CIFAR-10.
The number of total clients is fixed to $\nclients=50$.
}
\label{tbl:attacker_number}
\begin{center}
\scalebox{\scalefactor}{
\begin{tabular}{c|cc|cc|cc}
\toprule
$f$ & Multi-Krum & \proposedmethod (Multi-Krum) & Bulyan & \proposedmethod (Bulyan) & Median & \proposedmethod (Median) \\
\midrule
5 & 41.65 $\pm$ 1.78 & \textbf{61.24} $\pm$ 0.01 & 56.28 $\pm$ 1.44 & \textbf{58.27} $\pm$ 0.17 & 46.91 $\pm$ 1.36 & \textbf{57.69} $\pm$ 1.81 \\
15 & 10.00 $\pm$ 0.00 & \textbf{34.70} $\pm$ 0.28 & 10.00 $\pm$ 0.00 & \textbf{31.67} $\pm$ 0.19 & 18.85 $\pm$ 1.54 & \textbf{30.95} $\pm$ 0.42 \\
\midrule
$f$ & RFA & \proposedmethod (RFA) & DnC & \proposedmethod (DnC) & RBTM & \proposedmethod (RBTM) \\
\midrule
5 & 22.37 $\pm$ 1.00 & \textbf{58.06} $\pm$ 1.29 & 62.27 $\pm$ 0.04 & \textbf{63.14} $\pm$ 0.20 & 55.92 $\pm$ 0.10 & \textbf{59.72} $\pm$ 0.16 \\
15 & 16.16 $\pm$ 0.14 & \textbf{40.37} $\pm$ 0.26 & 57.28 $\pm$ 1.37 & \textbf{60.14} $\pm$ 1.64 & 34.93 $\pm$ 1.36 & \textbf{35.78} $\pm$ 1.51 \\
\bottomrule
\end{tabular}

}
\end{center}
\end{table*}

\begin{table*}[t]
\begin{center}
\caption{Accuracy (mean$\pm$std) of different defenses against LIE attack under different client numbers on CIFAR-10.}
\label{tbl:client_number}
\scalebox{\scalefactor}{

\begin{tabular}{c|cc|cc|cc}
\toprule
$n$ & Multi-Krum & \proposedmethod (Multi-Krum) & Bulyan & \proposedmethod (Bulyan) & Median & \proposedmethod (Median) \\
\midrule
75 & 28.72 $\pm$ 0.71 & \textbf{54.89} $\pm$ 0.16 & 23.37 $\pm$ 1.22 & \textbf{51.11} $\pm$ 0.00 & 44.89 $\pm$ 2.98 & \textbf{52.22} $\pm$ 1.64 \\
100 & 32.49 $\pm$ 1.22 & \textbf{56.51} $\pm$ 0.01 & 21.93 $\pm$ 0.55 & \textbf{46.49} $\pm$ 1.33 & 33.82 $\pm$ 0.21 & \textbf{46.12} $\pm$ 0.17 \\
\midrule
$n$ & RFA & \proposedmethod (RFA) & DnC & \proposedmethod (DnC) & RBTM & \proposedmethod (RBTM) \\
\midrule
75 & 16.89 $\pm$ 1.38 & \textbf{49.85} $\pm$ 0.06 & 59.31 $\pm$ 1.33 & \textbf{59.75} $\pm$ 0.42 & 45.06 $\pm$ 0.96 & \textbf{50.24} $\pm$ 0.31 \\
100 & 14.01 $\pm$ 1.34 & \textbf{49.85} $\pm$ 1.97 & 58.88 $\pm$ 1.45 & \textbf{59.61} $\pm$ 1.19 & 40.38 $\pm$ 0.48 & \textbf{47.02} $\pm$ 0.03 \\
\bottomrule
\end{tabular}

}
\end{center}
\end{table*}

\textbf{Evaluated attacks.} 
We consider six representative attacks
BitFlip \cite{allen2020safeguard}, LabelFlip \cite{allen2020safeguard}, LIE \cite{baruch2019lie}, Min-Max \cite{shejwalkar2021dnc}, Min-Sum \cite{shejwalkar2021dnc} and IPM \cite{xie2020ipm}.
The detailed hyperparameter setting of the attacks are shown in \cref{apptbl:attack_hyperparams} in \cref{appsec:setup}.

\textbf{Baselines.}
We consider six representative robust AGRs:
Multi-Krum \cite{blanchard2017krum},
Bulyan \cite{guerraoui2018bulyan},
Median \cite{yin2018mediantrmean},
RFA \cite{pillutla2019geometric},
DnC \cite{shejwalkar2021dnc},
RBTM \cite{el2021mda}.
We compare each AGR with its GAS variant and name them GAS (Multi-Krum), GAS (Bulyan), GAS (Median), GAS (RFA), GAS (DnC), and GAS (RBTM), respectively. 
The detailed hyperparameter settings of the robust AGRs are listed in \cref{apptbl:defense_hyperparams} in \cref{appsec:setup}.
We also compare our \proposedmethod against Bucketing \citep{karimireddy2022bucketing}.

\textbf{Evaluation.}
We use top-1 accuracy, i.e., the proportion of correctly predicted testing samples to total testing samples, to evaluate the performance of global models.
We run each experiment for five times and report the mean and standard deviation of the highest accuracy during the training process.

\textbf{Other settings.}
We utilize AlexNet \cite{krizhevsky2017alexnet}, SqueezeNet \cite{iandola2016squeezenet}, ResNet-18 \cite{he2016resnet} and a four-layer CNN \cite{caldas2018leaf} for CIFAR-10, CIFAR-100, ImageNet-12 and FEMNIST, respectively.
The number of Byzantine clients of all datasets is set to $f=0.2\cdot\nclients$.
We also consider up to $f=0.3\cdot\nclients$ Byzantine clients in the ablation study.
Please refer to \cref{apptable:default_setup} in \cref{appsec:setup} for more details.

\subsection{Experiment Results}
\label{subsec:results}

\textbf{Main results.}
\cref{tbl:main_experiments} illustrates the results of different defenses against popular attacks on CIFAR-10, CIFAR-100, ImageNet-12 and FEMNIST.
From these tables, we observe that:
\begin{itemize}
    \item[(1)] 
    Integrating current robust AGRs into our \proposedmethod generally outperform all their original versions on all datasets, which verifies the efficacy of our proposed \proposedmethod.
    For example, \proposedmethod improves the accuracy of Median by 15.93\% under Min-Sum attack on CIFAR-10.
    \item[(2)] The improvement of \proposedmethod (DnC) over DnC is relatively mild on CIFAR-10.
    Our interpretation is that when the dataset is relatively small and simple, DnC is capable of obtaining a rational gradient estimation.
    Nevertheless, on larger and more complex datasets, i.e., FEMNIST and ImageNet-12, DnC fails to achieve satisfactory performance under Byzantine attacks.
    \item[(3)] Although RFA collapses on FEMNIST, combining with our \proposedmethod can still improve it to satisfactory performance.
    Our illustration is that although the aggregated gradient of RFA deviates from the optimal gradient, it can still assist in identifying honest gradients when combined with \proposedmethod.
    As a result, \proposedmethod (RFA) is still effective on FEMNIST.
\end{itemize}

\textbf{GAS v.s. Bucketing.}
 We also compare our \proposedmethod method against Bucketing \cite{karimireddy2022bucketing} on CIFAR-10.
 For each robust AGR, we combine it with \proposedmethod or Bucketing separately and compare their performance.
 The results are posted in \cref{tbl:bucketing}.
 As shown in \cref{tbl:bucketing}, our \proposedmethod outperforms Bucketing in most cases. 
 Except for LIE attack, the test accuracy of \proposedmethod (Median) is slightly lower than Bucketing (Median).

\textbf{Number of sub-vectors.}
We vary sub-vector number $p$ across $\{100, 1000, 10000, 100000, 1000000, 2472266(d)\}$ under LIE attack on the heterogeneous CIFAR-10 dataset.
Other setups align with the main experiments.
The results are provided in \cref{tbl:ps}.
As shown in \cref{tbl:ps}, when $p$ increases, the accuracy of \proposedmethod first increases, then slightly drops.
Compared to \proposedmethod (Multi-Krum), \proposedmethod (Bulyan) demonstrates the best performance at a larger $p$ and declines more slowly as $p$ continues to increase.
These results imply that: (1) \proposedmethod with a moderate $p$ is more likely to achieve better performance;
(2) the best $p$ for different base AGRs differs.

\textbf{Performance of \proposedmethod when number of Byzantine clients $f$ is unknown.}
We run additional experiments to evaluate the performance of \proposedmethod when the number of Byzantine clients $f$ is unknown to the server.
In this case, \proposedmethod removes a fixed fraction of $\delta$ sampled clients in each communication round, where $\delta\in[0, 0.5)$ is the estimated ratio of Byzantine clients.
We test for $\delta=0.1, 0.3$ when there are 20\% Byzantine clients under LIE attack on CIFAR-10.
From results in \cref{tbl:unknown_f}, we can summarize that:
(1) When the server knows the number of Byzantine clients (i.e., $\delta$ is N/A), GAS achieves the best performance;
(2) When the number of Byzantine clients is unknown, the performance degradation of GAS is relatively mild;
(3) Compared to excluding fewer clients ($\delta=0.1$) from aggregation, the performance of GAS is generally better when excluding more clients ($\delta=0.3$). We hypothesize this is because gradient heterogeneity is more impactful than LIE attack. Therefore, excluding honest gradients ($\delta=0.3$) is more harmful to the performance of GAS compared to including Byzantine gradients ($\delta=0.1$).

\textbf{Results on different levels of non-IID.}
We discuss the impact of non-IID levels of data distributions.
We modify the concentration parameter $\beta$ to change the non-IID level.
A smaller $\beta$ implies a higher non-IID level.
\cref{tbl:non_iid} demonstrates the accuracy of different defenses under LIE attack on CIFAR-10 dataset across $\beta=\{0.3, 0.7\}$.
Other setups follow the default setup of the main experiments as illustrated in \cref{subsec:experimental_setups} and \cref{appsec:setup}.
As shown in \cref{tbl:non_iid},
all the existing AGRs achieve better performances than their original versions when combined with \proposedmethod, which validates the efficacy of our \proposedmethod under different non-IID levels.
Moreover, when the level of non-IID is higher, the improvement on robust AGRs is more significant.
The results further confirm that our \proposedmethod can overcome the failures aggravated under a higher non-IID level.

\textbf{Results on different number of Byzantine clients.}
We also conduct experiments across different number of Byzantine clients (the total number of clients $n$ is fixed).
Other setups follow the default setup of the main experiments in \cref{subsec:experimental_setups} and \cref{appsec:setup}.
\cref{tbl:attacker_number} demonstrates the results of different defenses under LIE attack across $\nattackers=\{5, 15\}$ Byzantine clients on CIFAR-10 dataset.
As shown in \cref{tbl:attacker_number}, our \proposedmethod outperforms the corresponding baselines across all Byzantine client numbers.

\textbf{Results on different number of clients.}
We further analyze the efficacy of our \proposedmethod under different number of clients.
We test the performance of different defenses under LIE attack across $\nclients=\{75, 100\}$ clients on CIFAR-10 dataset.
The number of Byzantine clients is set to $\nattackers=0.2\cdot\nclients$ correspondingly.
Other setups follow the default setup of the main experiments in \cref{subsec:experimental_setups} and \cref{appsec:setup}.
These results demonstrate that all the robust AGRs consistently outperform all their original versions when combined with our \proposedmethod, which validates that our \proposedmethod can effectively defend against Byzantine across different numbers of clients.

\section{Conclusion and Discussion}
\label{sec:conclusion}

In this work, we identify two main challenges of Byzantine robustness in the non-IID setting: the curse of dimensionality and gradient heterogeneity.
Robust AGRs that try to include all honest gradients in aggregation suffer from the curse of dimensionality.
Other robust AGRs that aggregate fewer gradients to get rid of Byzantines fail due to gradient heterogeneity.
Motivated by the above discoveries, we propose a novel \shorten (\proposedmethod) based approach that is compatible with most existing robust AGRs and overcomes the high dimensionality and gradient heterogeneity.
\proposedmethod splits each high-dimensional gradient into low-dimensional sub-vectors and detects Byzantine gradients with the sub-vectors to address the curse of dimensionality.
Then, \proposedmethod aggregates all the identified honest gradients to handle the gradient heterogeneity to alleviate the gradient heterogeneity issue.
We also provide a detailed convergence analysis of our proposed \proposedmethod. 
Empirical studies on four real-world datasets justify the efficacy of \proposedmethod.

\textbf{Discussion.}
In the first step of \proposedmethod, we use an equal splitting mechanism for splitting. In fact, there are many other mechanisms, e.g., split gradients by layer.
A future research direction is to discover more effective splitting mechanisms for \proposedmethod.
Note that our \proposedmethod can also be combined with adaptive client selection strategies \cite{wan2022adaptive} to achieve better Byzantine robustness.
We would discuss it more in our future work.

\section*{Acknowledgements}
This work is supported by the National Key R\&D Program of China (No.2022YFB3304100) and by the Zhejiang University-China Zheshang Bank Co., Ltd. Joint Research Center. This work is also sponsored by Sony AI.

\clearpage
\bibliography{references}
\bibliographystyle{icml2023}

\clearpage
\appendix
\onecolumn

\section{Setups for Experiments in \cref{sec:motivation}}
\label{appsec:motivation_setup}
The experiments are conducted on CIFAR-10 \cite{krizhevsky2009cifar}.

For both IID and non-IID settings, the number of clients is set to $\nclients=50$.
For IID data distribution, all 50,000 samples are randomly partitioned into 50 clients each containing 1,000 samples.
For non-IID data distribution, the samples are partitioned in a Dirichlet manner with concentration parameter $\beta=0.5$.
Please refer to \cref{subsec:experimental_setups} for the details of Dirichlet partition.

The number of Byzantine clients is set to $\nattackers=10$.
LIE \cite{baruch2019lie} attack with $z=1.5$ is considered.

We use AlexNet \cite{krizhevsky2017alexnet} as the model architecture.
The number of communication rounds is set to 500.
In each communication round, all clients participate in the training.

For local training,
the number of local epochs is set to 1,
batch size is set to 64,
the optimizer is set to SGD.
For SGD optimizer,
learning rate is set to 0.1,
momentum is set to 0.5,
weight decay coefficient is set to 0.0001.
We also adopt gradient clipping with clipping norm 2.

Six robust AGRs are considered: Bulyan \cite{guerraoui2018bulyan}, Median \cite{yin2018mediantrmean}, RBTM \cite{el2021mda}, Multi-Krum \cite{blanchard2017krum}, RFA \cite{pillutla2019geometric}, DnC \cite{shejwalkar2021dnc}

\section{Computation Cost of \proposedmethod}
\label{appsec:computation_cost}
We first give the computation cost of the proposed GAS method. The computation cost of GAS is closely related to the computation cost of the base robust AGR $\mathcal{A}$. We use $\mathtt{cost}_{\mathcal{A}}(d, n)$ to denote the computation cost of the base AGR $\mathcal{A}$ given $n$ gradients of dimensionality $d$.

Our GAS method has three steps: splitting, identification, and aggregation.

\begin{itemize}
\item[] \textbf{Splitting.} The splitting step is of complexity $\mathcal{O}(d)$;

\item[] \textbf{Identification.} The identification step consists of two parts: apply AGR $\mathcal{A}$ to sub-vectors ($\mathcal{O}(p \mathtt{cost}_{\mathcal{A}}(d/p, n))$) and compute identification scores ($\mathcal{O}(nd+np)$).

\item[] \textbf{Aggregation.} The complexity of aggregation step is $\mathcal{O}(n\log(n-f)+(n-f)d)$.
\end{itemize}

In summary, the overall complexity for GAS is $\mathcal{O}(d+p \mathtt{cost}_{\mathcal{A}}(d/p, n)+nd+ np+n\log(n-f)+ (n-f)d)= \mathcal{O}(n(d +\log(n-f))+p \mathtt{cost}_{\mathcal{A}}(d/p, n))$.

Then we analyze the computation cost of GAS $\mathcal{O}( n(d +\log(n-f))+p \mathtt{cost}_{\mathcal{A}}(d/p, n))$. Since $n\ll d$ [7], the first term $\mathcal{O}( n(d +\log(n-f)))\approx\mathcal{O}(d)\ll \Omega(d^2)$. The second term $ \mathcal{O}(p \mathtt{cost}_{\mathcal{A}}(d/p, n))$ relies on $\mathtt{cost}_{\mathcal{A}}(d/p, n)$, the cost of base AGR $\mathcal{A}$. The computation cost of popular AGRs are usually $\mathcal{O}(d/p)$ under assumption $n\ll d$ [7], e.g., Krum ($\mathcal{O}(n^2d/p)$), Bulyan ($\mathcal{O}(n^2d/p)$). Therefore, the second term usually satisfies $ \mathcal{O}(p \mathtt{cost}_{\mathcal{A}}(d/p, n))\approx \mathcal{O}(d)\ll \Omega(d^2)$. In summary, the computation cost of GAS is generally $\mathcal{O}(d)$ (consider only $d$ and omit $n$), which is much smaller than $\Omega(d^2)$.

\section{Convergence Analysis}

In this section, we provide the proof for our convergence results in \cref{prop:convergence} and the comparison of our convergence results with recent works.

We first restate the assumptions, the definition and the proposition for the integrity of this section.
\aspunbias*
\aspvar*
\aspdissim*
\asplipschitz*
\defbyzresil*
\propconvergence*

\subsection{Proof for \cref{prop:convergence}}
\label{appsec:proof}

\begin{lemma}
\label{lemma:resilience_random}
For positive integer $n$, $f\le n$ and real value $\lambda$, AGR $\agg$ is $(f,\lambda)$-resilient.
Then for any set of random variables $\{\bm{x}_1,\ldots,\bm{x}_n\}$ and $\gS\subseteq[n]$ of size $n-f$ that satisfies,
\begin{align}
\E[\|\bm{x}_i-\bm{x}_{i'}\|^2]\le\rho^2,\quad\forall i,i'\in\gH
\end{align}
we have
\begin{align}
\E[\|\agg(\bm{x}_1,\ldots,\bm{x}_n)-\bm{x}_{\gS}\|^2]\le4\lambda^2\cdot\frac{(n-f-1)^2}{n-f}\cdot\rho^2,
\end{align}
where $\bm{x}_{\gS}=\sum_{i\in\gS}\bm{x}_i/|\gS|$.
\end{lemma}

\begin{proof}
Since $\agg$ is $(f,\lambda)$-resilient, we have
\begin{align}
\label{appeq:resilience_1}
\E[\|\agg(\bm{x}_1,\ldots,\bm{x}_n)-\Bar{\bm{x}}_{\gS}\|^2]
\le\E[\lambda^2\max_{i,i'\in\gS}\|\bm{x}_i-\bm{x}_{i'}\|^2]
=\lambda^2\E[\max_{i,i'\in\gS}\|\bm{x}_i-\bm{x}_{i'}\|^2].
\end{align}

Then we bound $\max_{i,i'\in\gS}\|\bm{x}_i-\bm{x}_{i'}\|^2$ as follows.
\begin{align}
\label[ineq]{appeq:max_bound_1}
\max_{i,i'\in\gS}\|\bm{x}_i-\bm{x}_{i'}\|^2
\le{}&\max_{i,i'\in\gS}2(\|\bm{x}_i-\bm{x}_\gS\|^2+\|\bm{x}_\gS-\bm{x}_{i'}\|^2)\\
\le{}&\max_{i,i'\in\gS}2\|\bm{x}_i-\bm{x}_\gS\|^2+\max_{i,i'\in\gS}2\|\bm{x}_\gS-\bm{x}_{i'}\|^2\\
={}&4\max_{i\in\gS}\|\bm{x}_i-\bm{x}_\gS\|^2\\
\label[ineq]{appeq:max_bound_4}
\le{}&4\sum_{i\in\gS}\|\bm{x}_i-\bm{x}_\gS\|^2
\end{align}
Here \cref{appeq:max_bound_1} comes from the Cauchy inequality.

We further bound $\|\bm{x}_i-\bm{x}_\gS\|^2$ for all $i\in\gS$ as follows.
\begin{align}
\|\bm{x}_i-\bm{x}_\gS\|^2
={}&\frac{1}{(n-f)^2}\|\sum_{i'\in\gS\setminus\{i\}}(\bm{x}_i-\bm{x}_{i'})\|^2
\\\le{}&\frac{1}{(n-f)^2}\cdot(n-f-1)\sum_{i'\in\gS\setminus\{i\}}\|\bm{x}_i-\bm{x}_{i'}\|^2
\label{appeq:radius_bound_2}
\\={}&\frac{n-f-1}{(n-f)^2}\sum_{i'\in\gS\setminus\{i\}}\|\bm{x}_i-\bm{x}_{i'}\|^2.
\label{appeq:radius_bound_3}
\end{align}
Here \cref{appeq:radius_bound_2} comes from the Cauchy inequality.

Combine \cref{appeq:resilience_1,appeq:max_bound_4,appeq:radius_bound_3} and we have
\begin{align}
\E[\|\agg(\bm{x}_1,\ldots,\bm{x}_n)-\bm{x}_{\gS}\|^2]
\le{}&\lambda^2\E[\max_{i,i'\in\gS}\|\bm{x}_i-\bm{x}_{i'}\|^2]
\\\le{}&\lambda^2\E[4\sum_{i\in\gS}\|\bm{x}_i-\bm{x}_\gS\|^2]
\\\le{}&\lambda^2\E[4\sum_{i\in\gS}\frac{n-f-1}{(n-f)^2}\sum_{i'\in\gS\setminus\{i\}}\|\bm{x}_i-\bm{x}_{i'}\|^2]
\\\le{}&4\lambda^2\cdot\frac{n-f-1}{(n-f)^2}\sum_{i,i'\in\gS,i\ne i'}\E[\|\bm{x}_i-\bm{x}_{i'}\|^2]
\\\le{}&4\lambda^2\cdot\frac{n-f-1}{(n-f)^2}\cdot (n-f)(n-f-1)\rho^2
\\={}&4\lambda^2\cdot\frac{(n-f-1)^2}{n-f}\cdot\rho^2
\end{align}

\end{proof}

We state and prove the following lemma for the proof of \cref{lemma:aggregation_error}.
\begin{lemma}
\label{lemma:variance}
For any random vector $\vX$, we have
\begin{align}
\Var[\|\vX\|]\le\E\|\vX-\E\vX\|^2.
\end{align}
\end{lemma}
\begin{proof}
From the definition of variance, we have
\begin{align}
\Var[\|\vX\|]
={}&\E(\|\vX\|-\E\|\vX\|)^2\\
={}&\E(\|\vX\|-\|\E\vX\|)^2-(\|\E\vX\|-\E\|\vX\|)^2\\
\le{}&\E(\|\vX\|-\|\E\vX\|)^2\\
\le{}&\E\|\vX-\E\vX\|^2.
\end{align}
The second inequality comes from triangular inequality.
\end{proof}

\begin{lemma}[Aggregation error]
\label{lemma:aggregation_error}Suppose \cref{asp:unbias,asp:bounded_var,asp:dissimilarity} hold.
Given an $(f,\lambda)$-resilient robust AGR $\agg$, for any $t>0$, it satisfies
\begin{align}
\E[\|\aggsgrad{}-\Bar{\vg}\|^2]
\le\gO((\kappa^2+\sigma^2)(1+\frac{n-f+1}{p})(1+\lambda^2+\frac{1}{n-f})\frac{f^2}{(n-f)^2})
\end{align}
\end{lemma}

\begin{proof}

We rewrite $\hat{\vg}$ as follows.
\begin{align}
\hat{\vg}
=\frac{1}{n-f}\sum_{i\in\gI}\vg_i
=\frac{1}{n-f}(\sum_{h\in\Tilde{\honestclients}}\vg_h
    +\sum_{b\in\Tilde{\byzantineclients}}\vg_b)
=\frac{|\Tilde{\gH}|}{n-f}\vg_{\Tilde{\gH}}
    +\frac{|\Tilde{\gB}|}{n-f}\vg_{\Tilde{\gB}}.
\end{align}
Here $\Tilde{\honestclients}=\honestclients\cap\mathcal{I}$,
$\Tilde{\byzantineclients}=\byzantineclients\cap\mathcal{I}$,
and $\vg_{\gS}=\sum_{i\in\gS}\vg_i/|\gS|$ for all $\gS\in[n]$.

Then, we can bound $\E\left\|\hat{\vg}-\Bar{\vg}\right\|^2$ as follows
\begin{align}
\E[\|\hat{\vg}-\Bar{\vg}\|^2]
={}&\E[\|\frac{|\Tilde{\gH}|}{n-f}\vg_{\Tilde{\gH}}+\frac{|\Tilde{\gB}|}{n-f}\vg_{\Tilde{\gB}}-\Bar{\vg}\|^2]\\
={}&\E[\|\frac{|\Tilde{\gH}|}{n-f}(\vg_{\Tilde{\gH}}-\Bar{\vg})
    +\frac{|\Tilde{\gB}|}{n-f}(\vg_{\Tilde{\gB}}-\Bar{\vg})\|^2]\\
\le{}&\frac{2|\Tilde{\gH}|^2}{(n-f)^2}\E[\|\vg_{\Tilde{\honestclients}}-\Bar{\vg}\|^2]
    +\frac{2|\Tilde{\gB}|^2}{(n-f)^2}\E[\|\vg_{\Tilde{\byzantineclients}}-\Bar{\vg}\|^2].
\end{align}

We bound $\E[\|\vg_{\Tilde{\gH}}-\Bar{\vg}\|^2]$ as follows.
\begin{align}
\E[\|\vg_{\Tilde{\gH}}-\Bar{\vg}\|^2]
={}&\E[\|(\vg_{\Tilde{\gH}}-\Bar{\vg}_{\Tilde{\gH}})+(\Bar{\vg}_{\Tilde{\gH}}-\Bar{\vg})\|^2]\\
={}&\E[\|\vg_{\Tilde{\gH}}-\Bar{\vg}_{\Tilde{\gH}}\|^2]
    +\|\Bar{\vg}_{\Tilde{\gH}}-\Bar{\vg}\|^2\\
\label{eq:benign_dev_bound}
\le{}&\frac{\sigma^2}{|\Tilde{\gH}|}+\kappa^2,
\end{align}
where $\Bar{\vg}_{\Tilde{\gH}}=\E[\vg_{\Tilde{\gH}}]$.

Then we consder $\E[\|\vg_{\Tilde{\gB}}-\bar{\vg}\|^2]$.
According to the law of total expectation, we have
\begin{align}
\E[\|\vg_{\Tilde{\gB}}-\bar{\vg}\|^2]
=\sum_{\Tilde{f}=0}^f\E[\|\vg_{\Tilde{\gB}}-\bar{\vg}\|^2\mid|\Tilde{\gB}|=\Tilde{f}]\Pr(|\Tilde{\gB}|=\Tilde{f}).
\end{align}

For all parameter group $q\in[\ngroups]$ and $i,j\in\gH$, we have
\begin{align}
\E[\|\subsgrad{i}{q}-\subsgrad{j}{q}\|^2]
={}&\E[\|(\subsgrad{i}{q}-\subgrad{i}{q})+(\subgrad{i}{q}-\subgrad{}{q})+(\subgrad{}{q}-\subgrad{j}{q})+(\subgrad{j}{q}-\subsgrad{j}{q})\|^2]\\
\label{appeq:rho_3}
\begin{split}
={}&\E[\|\subsgrad{i}{q}-\subgrad{i}{q}\|^2]
+\|\subgrad{i}{q}-\subgrad{}{q}\|^2
+\|\subgrad{}{q}-\subgrad{j}{q}\|^2
+\E[\|\subgrad{j}{q}-\subsgrad{j}{q}\|^2]\\
&+2\langle\subgrad{i}{q}-\subgrad{}{q},\subgrad{}{q}-\subgrad{j}{q}\rangle
\end{split}
\\
\label[ineq]{appineq:rho_4}
\begin{split}
\le{}&\E[\|\subsgrad{i}{q}-\subgrad{i}{q}\|^2]
+\|\subgrad{i}{q}-\subgrad{}{q}\|^2
+\|\subgrad{}{q}-\subgrad{j}{q}\|^2
+\E[\|\subgrad{j}{q}-\subsgrad{j}{q})\|^2]\\
&+2\|\subgrad{i}{q}-\subgrad{}{q}\|\cdot\|\subgrad{}{q}-\subgrad{j}{q}\|
\end{split}
\\
\label[ineq]{appineq:rho_5}
\le{}&2\sigma^2+4\kappa^2
\end{align}
Here \cref{appeq:rho_3} is due to the independence of $\subsgrad{i}{q}$ and $\subsgrad{j}{q}$,
\cref{appineq:rho_4} comes from the Cauchy inequality,
and \cref{appineq:rho_5} follows \cref{asp:bounded_var,asp:dissimilarity}.

Then according to the \cref{lemma:resilience_random}, we have
\begin{align}
\label{appeq:group_byz_resilience}
\E[\|\subaggsgrad{q}-\subavgsgrad{q}\|^2]
\le c^2\max_{i,j\in\gH}\E[\|\subsgrad{i}{q}-\subsgrad{j}{q}\|^2]
\le c^2(2\sigma^2+4\kappa^2),
\end{align}
where $c^2=4\lambda^2(n-f-1)^2/(n-f)$.

For honest client $h$, the expectation of abnormal score $\groupscore{h}{q}$ from group $q$ can be bounded as follows.
\begin{align}
\E[\groupscore{h}{q}]
={}&\E[\|\subsgrad{h}{q}-\subaggsgrad{q}\|]\\
\label[ineq]{appeq:honest_mean_group_2}
\le{}&\E[\|\subsgrad{h}{q}-\subgrad{h}{q}\|
    +\|\subgrad{h}{q}-\Bar{\vg}^{(q)}\|
    +\|\Bar{\vg}^{(q)}-\vg^{(q)}\|
    +\|\vg^{(q)}-\subaggsgrad{q}\|]\\
={}&\E[\|\subsgrad{h}{q}-\subgrad{h}{q}\|]
    +\|\subgrad{h}{q}-\Bar{\vg}^{(q)}\|
    +\E[\|\Bar{\vg}^{(q)}-\vg^{(q)}\|]
    +\E[\|\vg^{(q)}-\subaggsgrad{q}\|]\\
\label[ineq]{appeq:honest_mean_group_4}
\le{}&\sqrt{\E[\|\subsgrad{h}{q}-\subgrad{h}{q}\|^2]}
    +\|\subgrad{h}{q}-\Bar{\vg}^{(q)}\|
    +\sqrt{\E[\|\Bar{\vg}^{(q)}-\vg^{(q)}\|^2]}
    +\sqrt{\E[\|\vg^{(q)}-\subaggsgrad{q}\|^2]}\\
\label[ineq]{appeq:honest_mean_group_5}
\le{}&(1+\frac{1}{\sqrt{n-f}})\sigma+\kappa+c\sqrt{2\sigma^2+4\kappa^2}.
\end{align}
Here \cref{appeq:honest_mean_group_2} is a result of triangular inequality,
\cref{appeq:honest_mean_group_4} comes from Cauchy inequality,
and \cref{appeq:honest_mean_group_5} is a combined result of \cref{appeq:group_byz_resilience} and \cref{asp:bounded_var,asp:dissimilarity}.

The variance of $\groupscore{h}{q}$ can be bounded as follows.
\begin{align}
\Var[\groupscore{h}{q}]
={}&\E[(\groupscore{h}{q})^2]-(\E[\groupscore{h}{q}])^2\\
\le{}&\E[(\groupscore{h}{q})^2]\\
={}&\E[\|\subsgrad{h}{q}-\subaggsgrad{q}\|^2]\\
={}&\E[\|(\subsgrad{h}{q}-\subgrad{h}{q})
+(\subgrad{h}{q}-\subgrad{}{q})
+(\subgrad{}{q}-\subavgsgrad{q})
+(\subavgsgrad{q}-\subaggsgrad{q})\|^2]\\
\label[ineq]{appeq:honest_var_group_split_5}
\le{}&4\E[\|\subsgrad{h}{q}-\subgrad{h}{q}\|^2
+\|\subgrad{h}{q}-\subgrad{}{q}\|^2
+\|\subgrad{}{q}-\subavgsgrad{q}\|^2
+\|\subavgsgrad{q}-\subaggsgrad{q}\|^2].
\end{align}
Here \cref{appeq:honest_var_group_split_5} is a result of Cauchy inequality.

We bound $\E[\|\subgrad{}{q}-\subavgsgrad{q}\|^2]$ as follows.
\begin{align}
\E[\|\subgrad{}{q}-\subavgsgrad{q}\|^2]
={}&\E[\|\frac{1}{n-f}\sum_{i\in\honestclients}(\subgrad{i}{q}-\subsgrad{i}{q})\|^2]\\
\label{appeq:honest_estimator_var_2}
={}&\frac{1}{(n-f)^2}\sum_{i\in\honestclients}\E[\|\subgrad{i}{q}-\subsgrad{i}{q}\|^2]\\
\label[ineq]{appeq:honest_estimator_var_3}
\le{}&\frac{1}{(n-f)^2}\sum_{i\in\honestclients}\sigma^2\\
\label{appeq:honest_estimator_var_4}
={}&\frac{\sigma^2}{n-f}
\end{align}
Here \cref{appeq:honest_estimator_var_2} comes from the independence of minibatches sampling across different clients, and \cref{appeq:honest_estimator_var_3} is a result of \cref{asp:bounded_var}.

Applying \cref{asp:bounded_var,asp:dissimilarity,appeq:group_byz_resilience,appeq:honest_estimator_var_4} to \cref{appeq:honest_var_group_split_5}, we have
\begin{align}
\Var[\groupscore{h}{q}]
\le{}&4(\sigma^2+\kappa^2+\frac{\sigma^2}{n-f}+c(2\sigma^2+4\kappa^2))\\
\label{appeq:honest_var_group}
={}&(4+8c^2+\frac{4}{n-f})\sigma^2+(4+16c^2)\kappa^2.
\end{align}

According to \cref{appeq:honest_mean_group_5,appeq:honest_var_group}, we can bound the expectation and variance of total abnormal score $\clientscore{h}$ of an honest client $h$.
\begin{gather}
\label{appeq:honest_mean}
\E[\clientscore{h}]
=\E[\sum_{q=1}^\ngroups\groupscore{h}{q}]
\le \ngroups(\sigma+\kappa+c\sqrt{2\sigma^2+4\kappa^2}):=A,\\
\label{appeq:honest_var}
\Var[\clientscore{h}]
=\sum_{q=1}^\ngroups\Var[\groupscore{h}{q}]
\le \ngroups((4+8c^2+\frac{4}{n-f})\sigma^2+(4+16c^2)\kappa^2):=B.
\end{gather}
Here the addictive property of variance is a result of the independence of group abnormal scores $\{\groupscore{h}{q}\mid q\in[\ngroups]\}$, which comes from the independence of components in a gradient \cite{yang2017independence}.

From Chebyshev's inequality, for any $\Delta_h>0$ and honest client $h\in[\nclients]\setminus\byzantineclients$, we have
\begin{align}
\label{appeq:honest_chebyshev}P(\clientscore{h}<\E[\clientscore{h}]+\Delta_h)
\ge1-\frac{\Var[\clientscore{h}]}{\Delta_h^2}.
\end{align}

Consider the expectation of abnormal score $\groupscore{b}{q}$ from group $q$ for Byzantine client $b\in\byzantineclients$
\begin{align}
\E[\groupscore{b}{q}]
={}&\E[\|\subsgrad{b}{q}-\subaggsgrad{q}\|]\\
={}&\E[\|(\subsgrad{b}{q}-\subgrad{}{q})-(\subaggsgrad{q}-\subgrad{}{q})\|]\\
\ge{}&\E[\|\subsgrad{b}{q}-\subgrad{}{q}\|-\|\subaggsgrad{q}-\subgrad{}{q}\|]\\
\ge{}&\E[\|\subsgrad{b}{q}-\subgrad{}{q}\|
-(\|\subaggsgrad{q}-\subsgrad{}{q}\|
+\|\subsgrad{}{q}-\subgrad{}{q}\|)]\\
\ge{}&\|\subsgrad{b}{q}-\subgrad{}{q}\|
-(\sqrt{\E[\|\subaggsgrad{q}-\subgrad{}{q}\|^2]}
+\sqrt{\E[\|\subsgrad{}{q}-\subgrad{}{q}\|^2}])\\
\label{appeq:byz_mean_group}
\ge{}&\delta_b-c\sqrt{2\sigma^2+4\kappa^2}-\frac{\sigma}{\sqrt{n-f}}
 \end{align}
where $\delta_b=\|\subsgrad{b}{q}-\subgrad{}{q}\|$ is the expected deviation of  Byzantine client $b$ from the average of honest gradients.
Here the first and second inequalities come from triangular inequality,
the third inequality is based on Cauchy inequality,
and the 4-th inequality is a combined result of \cref{appeq:group_byz_resilience,appeq:honest_estimator_var_4}.

The variance of abnormal score $\groupscore{b}{q}$ can be bounded as follows.
\begin{align}
\Var[\groupscore{b}{q}]
={}&\Var[\|\subsgrad{b}{q}-\subaggsgrad{q}\|]\\
\le{}&\E[\|\subsgrad{b}{q}-\subaggsgrad{q}-\E[\subsgrad{b}{q}-\subaggsgrad{q}]\|^2]\\
={}&\E[\|(\subsgrad{b}{q}-\E[\subsgrad{b}{q}])-(\subaggsgrad{q}-\E[\subaggsgrad{q}])\|^2]\\
\le{}&2\E[\|\subsgrad{b}{q}-\E\subsgrad{b}{q}\|^2+\|\subaggsgrad{q}-\E\subaggsgrad{q}\|^2]\\
\label{appeq:byz_var_group_split}
={}&2\E\|\subsgrad{b}{q}-\E\subsgrad{b}{q}\|^2
+2\E\|\subaggsgrad{q}-\E\subaggsgrad{q}\|^2
\end{align}
The first inequality results from \cref{lemma:variance},
and the second inequality comes from Cauchy inequality.

We bound $\|\subaggsgrad{q}-\E\subaggsgrad{q})\|$ as follows.
\begin{align}
\E\|\subaggsgrad{q}-\E\subaggsgrad{q}\|
={}&\E\|(\subaggsgrad{q}-\subavgsgrad{q})
+(\subavgsgrad{q}-\E\subavgsgrad{q})
-\E[\subaggsgrad{q}-\subavgsgrad{q}]\|^2]\\
\le{}&3\E[\|\subaggsgrad{q}-\subavgsgrad{q}\|^2
+\|\subavgsgrad{q}-\E\subavgsgrad{q}\|^2
+\|\E[\subaggsgrad{q}-\subavgsgrad{q}]\|^2]\\
\le{}&6\E\|\subaggsgrad{q}-\subavgsgrad{q}\|^2
+3\E\|\subavgsgrad{q}-\E\subavgsgrad{q}\|^2\\
\le{}&48(\sigma^2+\kappa^2)+\frac{3\sigma^2}{n-f}\\
\label{appeq:agg_bound}
={}&(48+\frac{3\sigma^2}{n-f})\sigma^2+48\kappa^2
\end{align}

Apply \cref{appeq:agg_bound} to \cref{appeq:byz_var_group_split}, we have
\begin{align}
\label{appeq:byz_var_group}
\Var[\groupscore{b}{q}]
\le2\sigma_b^2+(96+\frac{6}{n-f})\sigma^2+96\kappa^2,
\end{align}
where $\sigma_b^2=\E\|\subsgrad{b}{q}-\E\subsgrad{b}{q}\|^2$ is the variance.

Similar to \cref{appeq:honest_mean,appeq:honest_var}, we utilize \cref{appeq:byz_mean_group,appeq:byz_var_group} to bound the expectation and variance of total abnormal score $\clientscore{b}$ of a byzantine client $b$.
\begin{gather}
\label{appeq:byz_mean}
\E[\clientscore{b}]
=\E[\sum_{q=1}^\ngroups\groupscore{b}{q}]
\ge \ngroups(\delta_b-2\sqrt{2}c\sqrt{2\sigma^2+4\kappa^2}-\frac{\sigma}{\sqrt{n-f}}):=C,\\
\label{appeq:byz_var}
\Var[\clientscore{b}]
=\sum_{q=1}^\ngroups\Var[\groupscore{b}{q}]
\le \ngroups(2\text{const}+(96+\frac{6}{n-f})\sigma^2+96\kappa^2):=D
\end{gather}
where $\delta_b=\E\|\sgrad{b}{}-\grad{}{}\|$.
According to \citet{shejwalkar2021dnc}, $\sigma_b^2$ is bounded, i.e., $\sigma_b^2\le\text{const}$.

Similarly, we apply Chebyshev's inequality to the abnormal score of a Byzantine client $b\in\byzantineclients$.
\begin{align}
\label{appeq:byz_chebyshev}
\Pr(\clientscore{b}\ge\E[\clientscore{b}]-\Delta_b)\ge1-\frac{\Var[\clientscore{b}]}{\Delta_b^2},\quad b\in\byzantineclients.
\end{align}

Combine \cref{appeq:honest_mean,appeq:honest_var,appeq:honest_chebyshev}, and take $\Delta_h=(C-A)/(1+\sqrt{D/B})$, we have
\begin{align}
\Pr(\clientscore{h}<\frac{\sqrt{D}A+\sqrt{B}C}{\sqrt{B}+\sqrt{D}})
={}&\Pr(\clientscore{h}<A+\Delta_h)\\
\ge{}&\Pr(\clientscore{h}<\E[\clientscore{h}]+\Delta_h)\\
\ge{}&1-\frac{\Var[\clientscore{h}]}{\Delta_h^2}\\
\ge{}&1-\frac{B}{\Delta_h^2}\\
={}&1-\frac{(\sqrt{B}+\sqrt{D})^2}{(C-A)^2},
\end{align}

Combine \cref{appeq:byz_mean,appeq:byz_var,appeq:byz_chebyshev}, and take $\Delta_b=(C-A)/(1+\sqrt{B/D})$, we have
\begin{align}
\label{appeq:prob_byzantine}
\Pr(\clientscore{b}\ge\frac{\sqrt{D}A+\sqrt{B}C}{\sqrt{B}+\sqrt{D}})
\ge{}&\Pr(\clientscore{b}>C-\Delta_b)\\
\ge{}&\Pr(\clientscore{b}>\E[\clientscore{b}]-\Delta)\\
\ge{}&1-\frac{\Var[\clientscore{b}]}{\Delta^2}\\
\ge{}&1-\frac{D}{\Delta_b^2},\\
={}&1-\frac{(\sqrt{B}+\sqrt{D})^2}{(C-A)^2},
\end{align}

Then consider the probability a Byzantine $b$ is selected,
\begin{align}
\Pr(b\in\Tilde{\gB})
={}&1-\Pr(b\notin\Tilde{\gB})\\
\le{}&1-\Pr(\clientscore{h}<\frac{\sqrt{D}A+\sqrt{B}C}{\sqrt{B}+\sqrt{D}}, \forall h\in\honestclients, \clientscore{b}>\frac{\sqrt{D}A+\sqrt{B}C}{\sqrt{B}+\sqrt{D}})\\
={}&\Pr(\clientscore{h}\ge\frac{\sqrt{D}A+\sqrt{B}C}{\sqrt{B}+\sqrt{D}}, \forall h\in\honestclients\text{ or }\clientscore{b}<\frac{\sqrt{D}A+\sqrt{B}C}{\sqrt{B}+\sqrt{D}})\\
\le{}&\sum_{h\in\honestclients}\Pr(\clientscore{h}\ge\frac{\sqrt{D}A+\sqrt{B}C}{\sqrt{B}+\sqrt{D}})+\Pr(\clientscore{b}<\frac{\sqrt{D}A+\sqrt{B}C}{\sqrt{B}+\sqrt{D}})\\
\le{}&(n-f+1)\cdot\frac{(\sqrt{B}+\sqrt{D})^2}{(C-A)^2}
\end{align}
Solve $(n-f+1)\cdot(\sqrt{B}+\sqrt{D})^2/(C-A)^2\le\varepsilon$, we have
\begin{align}
\begin{split}
\E[\|\sgrad{b}{}-\grad{}{}\|]
\ge{}&(1+\frac{1}{\sqrt{n-f}})\sigma+\kappa+2c\sqrt{2\sigma^2+4\kappa^2}\\
&+\sqrt{\frac{n-f+1}{p\varepsilon}}(
\sqrt{(4+16c^2+\frac{4}{n-f})\sigma^2+(4+8c^2)\kappa^2}\\
&+\sqrt{2\text{const}+(96+\frac{6}{n-f})\sigma^2+96\kappa^2}
)
\end{split}
\end{align}
which implies that the Byzantine gradients that deviate from the optimal gradient will be filtered by \proposedmethod.

Therefore, for all $b\in\Tilde{\gB}$,
\begin{align}
\label{eq:byz_dev_bound}
\E[\|\vg_{\Tilde{\byzantineclients}}-\Bar{\vg}\|^2]
\le\gO((\kappa^2+\sigma^2)(1+\lambda^2+\frac{1}{n-f})(1+\frac{n-f+1}{p})):=C_1^2
\end{align}
The elimination of $\varepsilon$ is due to the sub-Gaussian property of $\vg_{\Tilde{\byzantineclients}}-\Bar{\vg}$, which comes from the Gaussian property of benign gradients.

Combine \cref{eq:benign_dev_bound,eq:byz_dev_bound}, $\E[\|\hat{\vg}-\Bar{\vg}\|^2]$ is finally bounded as follows.
\begin{align}
&\E[\|\hat{\vg}-\Bar{\vg}\|^2]\\
\le{}&\frac{|\Tilde{\gH}|^2}{(n-f)^2}(\sigma^2/\Tilde{h}+\kappa^2)+\frac{|\Tilde{\gB}|^2}{(n-f)^2}C_1^2\\
\label{eq:dev_bound}
\le{}&\frac{(n-2f)^2}{(n-f)^2}(\sigma^2/(n-2f)+\kappa^2)
    +\frac{f^2}{(n-f)^2}C_1^2,\\
={}&\gO((\kappa^2+\sigma^2)(1+\frac{n-f+1}{p})(1+\lambda^2+\frac{1}{n-f})\frac{f^2}{(n-f)^2})
\end{align}
which completes the proof.
\end{proof}

\subsubsection{Proof For The Main Proposition}

\begin{proof}
According to the Lipschitz property of loss function $\loss$, we have
\begin{align}
\label{eq:one_step_lipschitz}
\loss(\param[t])-\loss(\param[t+1])
\ge{}&\left<\nabla\loss(\param[t]), \param[t]-\param[t+1]\right>-\frac{L}{2}\|\param[t]-\param[t+1]\|^2.
\end{align}
Since $\param[t]-\param[t+1]=\nabla\loss(\param[t])+(\hat{\vg}^t-\nabla\loss(\param[t]))$, we can write \cref{eq:one_step_lipschitz} as follows
\begin{equation}
\label{eq:one_step_expand}
\begin{aligned}
\loss(\param[t])-\loss(\param[t+1])
\ge{}&(\eta-\frac{L}{2}\eta^2)\|\nabla\loss(\param[t])\|^2\\
    &+(\eta-\frac{L}{2}\eta^2)\langle\nabla\loss(\param[t]), \aggsgrad{t}-\nabla\loss(\param[t])\rangle\\
    &-\frac{L}{2}\eta^2\|\hat{\vg}^t-\nabla\loss(\param[t])\|^2.
\end{aligned}
\end{equation}
Then, we bound inner product term $\langle\nabla\loss(\param[t]), \aggsgrad{t}-\nabla\loss(\param[t])\rangle$.
\begin{align}
|\nabla\loss(\param[t]), \aggsgrad{t}-\nabla\loss(\param[t])\rangle|
\le{}&\|\langle\nabla\loss(\param[t])\|\cdot\|\aggsgrad{t}-\nabla\loss(\param[t])\|\\
\label{eq:ip_bound}
\le{}&\frac{1}{2}\|\nabla\loss(\param[t])\|^2
    +2\|\aggsgrad{t}-\nabla\loss(\param[t])\|^2
\end{align}
Combine \cref{eq:one_step_expand,eq:ip_bound} and we have
\begin{align}
\begin{split}
\loss(\param[t])-\loss(\param[t+1])
\ge{}&(\eta-\frac{L}{2}\eta^2)\|\nabla\loss(\param[t])\|^2\\
    &+(\eta-\frac{L}{2}\eta^2)\cdot-(\frac{1}{2}\|\nabla\loss(\param[t])\|^2
    +2\|\aggsgrad{t}-\nabla\loss(\param[t])\|^2)\\
    &-\frac{L}{2}\eta^2\|\hat{\vg}^t-\nabla\loss(\param[t])\|^2
\end{split}\\
\label{eq:one_step_expand_final}
={}&(\frac{1}{2}\eta-\frac{L}{4}\eta^2)\|\nabla\loss(\param[t])\|^2
    -(2\eta-\frac{L}{2}\eta^2)\|\hat{\vg}^t-\nabla\loss(\param[t])\|^2
\end{align}
Take the expectation on both sides of \cref{eq:one_step_expand_final}, we have
\begin{align}
\label[ineq]{eq:one_step_expectation}
\E[\loss(\param[t])-\loss(\param[t+1])]
\ge{}&(\frac{1}{2}\eta-\frac{L}{4}\eta^2)\E[\|\nabla\loss(\param[t])\|^2]
    -(2\eta-\frac{L}{2}\eta^2)\E[\|\hat{\vg}^t-\nabla\loss(\param[t])\|^2].
\end{align}


Apply \cref{lemma:aggregation_error} to \cref{eq:one_step_expectation} and sum over $t=0, 1, \ldots, T-1$, then we have
\begin{align}
\E[\loss(\param[0])-\loss(\param[T])]
\ge{}&(\frac{1}{2}\eta-\frac{L}{4}\eta^2)\sum_{t-1}^T\E[\|\nabla\loss(\param[t])\|^2]
    -T(\frac{1}{2}\eta-\frac{L}{2}\eta^2)C^2.
\end{align}
where $C^2=\gO((\kappa^2+\sigma^2)(1+(n-f+1)/p)(1+\lambda^2+1/(n-f))\frac{f^2}{(n-f)^2})$

Take $\eta=1/2L$, and consider that the loss function is generally non-negative, e.g., cross-entropy loss, $\ell_2$ loss,
\begin{align}
\E[\loss(\param[0])]
\ge{}&\frac{3}{16L}\sum_{t-1}^T(\E[\|\nabla\loss(\param[t])\|^2]
    -\frac{2}{3}C^2),
\end{align}
which completes the proof.
\end{proof}

\subsection{Comparasion of Our Convergence Results with Recent Works}
\label{appsec:comparaion}

Recent works \cite{karimireddy2022bucketing,yu2022secure,el2021mda,allen2020safeguard} also analyze the convergence of Byzantine-robust FL in the non-IID setting.
We compare our convergence results with them.

\textbf{Similarities.}
We all guarantee that we can reach an approximate optimal point after a certain number of communication rounds.
Moreover, we all admit that convergence in the presence of Byzantine clients may be impossible due to non-IID data, i.e., $\|\nabla\loss(\param)\|$ may never decrease to zero. 

\textbf{The difference from \citet{karimireddy2022bucketing}.}
Our result is orthogonal to one in \citet{karimireddy2022bucketing} since our \proposedmethod method is orthogonal to the Bucketing scheme proposed by \citet{karimireddy2022bucketing}:
we focus on how gradient splitting can alleviate the curse of dimensionality and gradient heterogeneity at the same time while \cite{karimireddy2022bucketing} considers how partitioning gradients into buckets can help with non-IID data.
In fact, we can obtain a better convergence result by combining our method with Bucketing scheme \cite{karimireddy2022bucketing}. The result would enjoy the strengths of both our \proposedmethod method and Bucketing scheme: (1) free from the curse of dimensionality; (2) handle gradient heterogeneity that comes from non-IID data; (3) the variance term diminishes when there is no Byzantine client.

\textbf{The difference from \citet{el2021mda}.}
Technically, our result is orthogonal from one in \citet{el2021mda}.
\citet{el2021mda} consider how to improve robust AGRs to achieve optimal Byzantine resilience.
We focus on how to handle the high-dimension nature of gradients.
Moreover, \citet{el2021mda} focus on decentralized FL with a server and provide an order optimal upper bound.
However, this strong result requires a Byzantine ratio lower than $1/3$.
By contrast, we consider a centralized FL setting and only assume the Byzantine ratio to be lower than $1/2$.

\textbf{The difference from \citet{peng2022vr}.}
\citet{peng2022vr} consider how client variance reduction and robust AGRs can jointly improve Byzantine resilience. And we concentrate more on gradient dimensions
\citet{peng2022vr} consider an ideal case where the objective function is strongly convex, while we consider a more general non-convex case.

\textbf{The difference from \citet{yu2022secure}.}
We considered different settings. We consider standard federated learning with a central server and \cite{yu2022secure} considers distributed optimization without a central server. Besides, the convergence analysis is based on different assumptions.

\begin{itemize}
\item 
\citet{yu2022secure} assume the strong convexity of the loss function (Assumption 3) while we do not. This assumption is restrictive since global models are neural networks in practical settings.
\item 
\citet{yu2022secure} do not assume uniformly bounded gradient differences but assume a common global minimizer. Instead, they assume a common minimizer among different agents (clients).
\end{itemize}

Due to different settings and assumptions, our convergence results are different.
\citet{yu2022secure} guarantee almost sure convergence while we ensure that we can approach an approximate optimal parameter.
Note that our upper bound matches the lower bound in \cite{karimireddy2022bucketing}.

\section{Experiment Setup}
\label{appsec:setup}

\subsection{Setup for Main Experiments in Section \ref{sec:experiments}}

\textbf{Data distribution.}
For CIFAR-10, CIFAR-100 \cite{krizhevsky2009cifar} and ImageNet-12 \cite{li2021imagenet12}, we use Dirichlet distribution to generate non-IID data by following \citet{yurochkin2019bayesian, li2021federated}.
In particular, for each client $i$, we sample $p_i^{\target} \sim \text{Dir}(\beta)$ and allocate a $p^\target_i$ proportion of the data of label $\target$ to client $i$, where $\text{Dir}(\beta)$ represents the Dirichlet distribution with a concentration parameter $\beta$.
We follow \citet{li2021federated} and set the number of clients $\nclients=50$ and the concentration parameter $\beta=0.5$ as default.

\textbf{Other setups.}
The setups for datasets FEMNIST \cite{caldas2018leaf}, CIFAR-10 \cite{krizhevsky2009cifar}, CIFR-100 \cite{krizhevsky2009cifar} and ImageNet-12 \cite{russakovsky2015imagenet} are listed in below \cref{apptable:default_setup}.
\begin{table}[h]
\caption{Default experimental settings for FEMNIST, CIFAR-10, CIFAR-100 and ImageNet-12.}
\label{apptable:default_setup}
\begin{center}
\resizebox{\textwidth}{!}{
\begin{tabular}{lllll}
\toprule 
Dataset & FEMNIST & CIFAR-10 & CIFAR-100 & ImageNet-12 \\
Architecture & \makecell[l]{CNN \\ \cite{caldas2018leaf}} & \makecell[l]{AlexNet \\ \cite{krizhevsky2017alexnet}} & \makecell[l]{SqueezeNet \\ \cite{iandola2016squeezenet}} & \makecell[l]{ResNet-18 \\ \cite{he2016resnet}} \\
\midrule
\makecell[l]{\# Communication rounds} & 1000 & 200 & 400 & 200 \\
Client sample ratio & 0.005 & 0.1 & 0.1 & 0.1 \\
\midrule
\# Local epochs & 1 & 5 & 1 & 1 \\
Optimizer & SGD & SGD & SGD & SGD \\
Batch size & 64 & 64 & 64 & 128 \\
Learning rate & 0.5 & 0.1 & 0.1 & 0.1 \\
Momentum & 0.5 & 0.5 & 0.5 & 0.9 \\
Weight decay & 0.0001 & 0.0001 & 0.0001 & 0.0001 \\
Learning rate decay & No & No & No & \makecell[l]{Reduce to 0.01 \\ after 100-th \\ communication \\ round} \\
Gradient clipping & Yes & Yes & Yes & Yes \\
Clipping norm & 2 & 2 & 2 & 2 \\
\bottomrule
\end{tabular}
}
\end{center}
\end{table}

The hyperparameters of six attacks: BitFlip \cite{allen2020safeguard}, LabelFlip \cite{allen2020safeguard}, LIE \cite{baruch2019lie}, Min-Max \cite{shejwalkar2021dnc}, Min-Sum \cite{shejwalkar2021dnc}, IPM \cite{xie2020ipm}, are listed in \cref{apptbl:attack_hyperparams} below.
\clearpage
\begin{table}[H]
\caption{The hyperparameters of six attacks. N/A indicates that the attack has no hyperparameters that need to be set.}
\label{apptbl:attack_hyperparams}
\begin{center}
\begin{tabular}{ll}
\toprule
Attacks & Hyperparameters \\
\midrule
BitFlip & N/A \\
LabelFlip & N/A \\
LIE & $z=1.5$ \\
Min-Max & $\gamma_{\text{init}}=10,\tau=1\times10^{-5}$, $\vdelta$: coordinate-wise standard deviation\\
Min-Sum & $\gamma_{\text{init}}=10,\tau=1\times10^{-5}$, $\vdelta$: coordinate-wise standard deviation \\
IPM & $\text{\# eval}=2$\\
\bottomrule
\end{tabular}
\end{center}
\end{table}

The hyperparameters of six robust AGRs: 
Multi-Krum \cite{blanchard2017krum},
Bulyan \cite{guerraoui2018bulyan},
Median \cite{yin2018mediantrmean},
RFA \cite{pillutla2019geometric},
DnC \cite{shejwalkar2021dnc},
RBTM \cite{el2021mda},
are listed in \cref{apptbl:defense_hyperparams} below.
\begin{table}[H]
\caption{The default hyperparameters of the AGRs. N/A indicates that the robust AGR has no hyperparameters that need to be set.}
\label{apptbl:defense_hyperparams}
\begin{center}
\begin{tabular}{ll}
\toprule
AGRs & Hyperparameters \\
\midrule
Multi-Krum & N/A \\
Bulyan & N/A \\
Median & N/A \\
RFA & $T=3$ \\
DnC & $c=4, \textsf{niters}=1, b=10000$ \\
RBTM & N/A \\
\bottomrule
\end{tabular}
\end{center}
\end{table}

\section{\proposedmethod mitigates the deviation of aggregated gradients}

In \cref{sec:theoretical}, we claim that our \proposedmethod approach can reduce the deviation of aggregated gradient $\aggsgrad{}$ from the average of honest gradients $\avgsgrad{}$.
To verify this fact, we compare the deviation of the aggregated gradient of different defenses and their \proposedmethod variants in \cref{fig:deviation}.
In particular, we use $\|\aggsgrad{}-\avgsgrad{}\|$, the distance between the aggregated gradient $\aggsgrad{}$ and the average of honest gradients $\avgsgrad{}$ to measure the deviation degree.
As shown in \cref{fig:deviation}, the gradient deviation degree of \proposedmethod-enhanced defenses is much lower than their original versions as expected, which validates that our \proposedmethod can mitigate the gradient deviation.
\begin{figure}[H]
\centering
\includegraphics[width=0.3\textwidth]{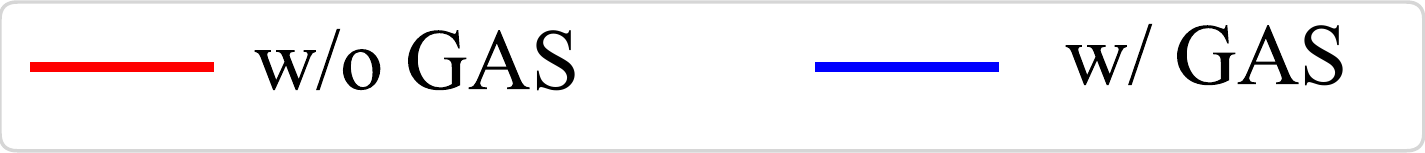}
\scalebox{0.7}{
\begin{subfigure}[]{\textwidth}
\centering
\includegraphics[width=0.3\textwidth]{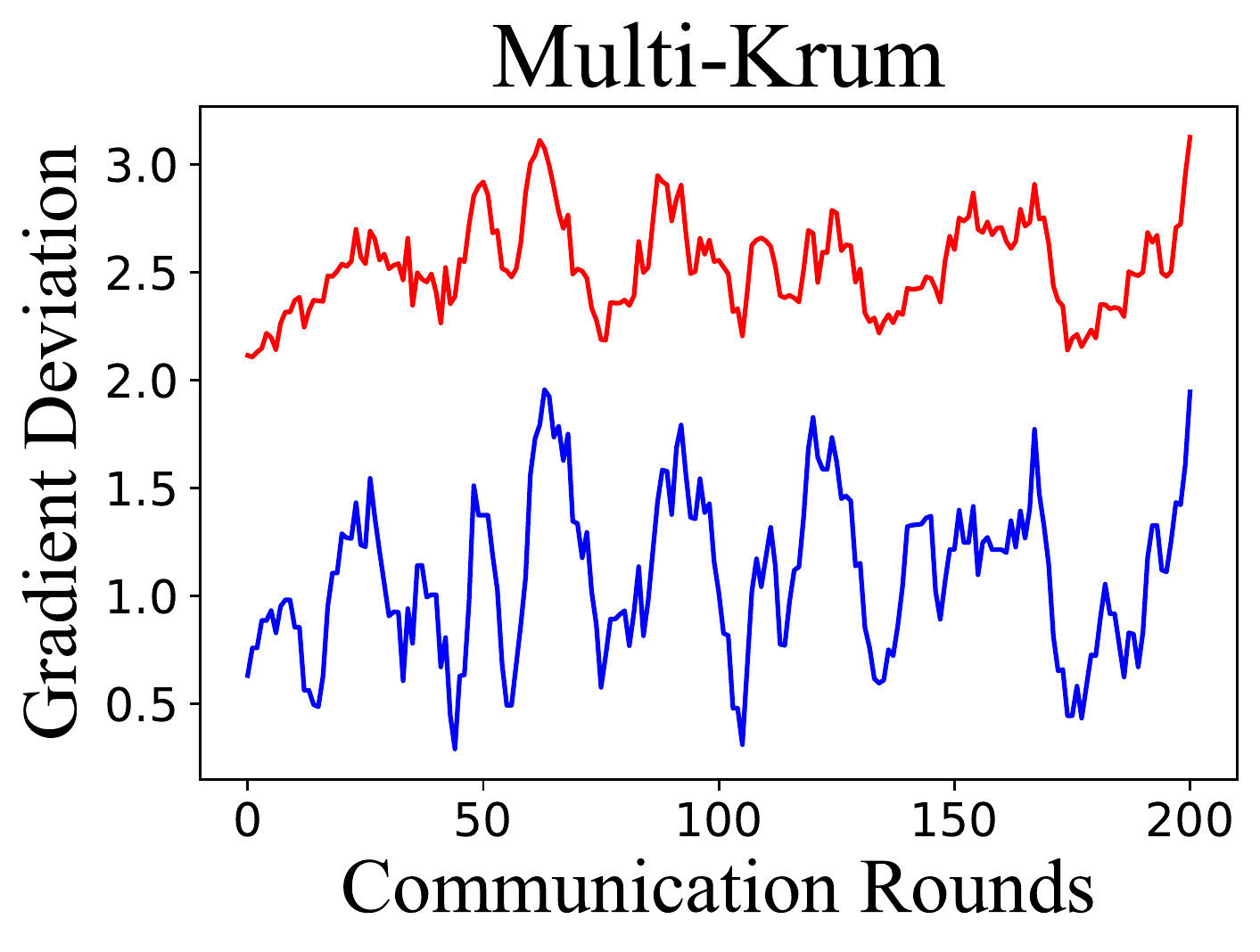}
\includegraphics[width=0.3\textwidth]{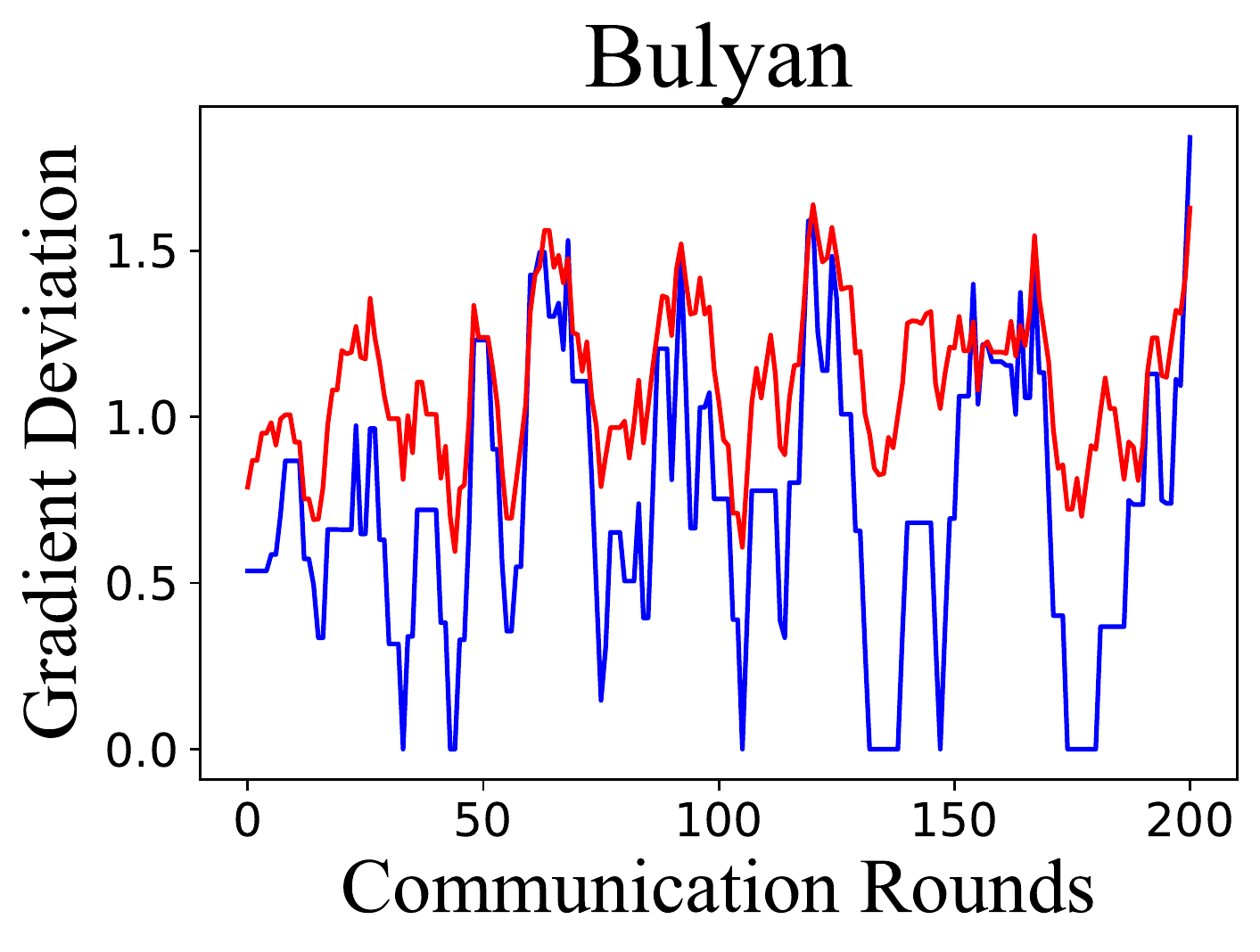}
\includegraphics[width=0.3\textwidth]{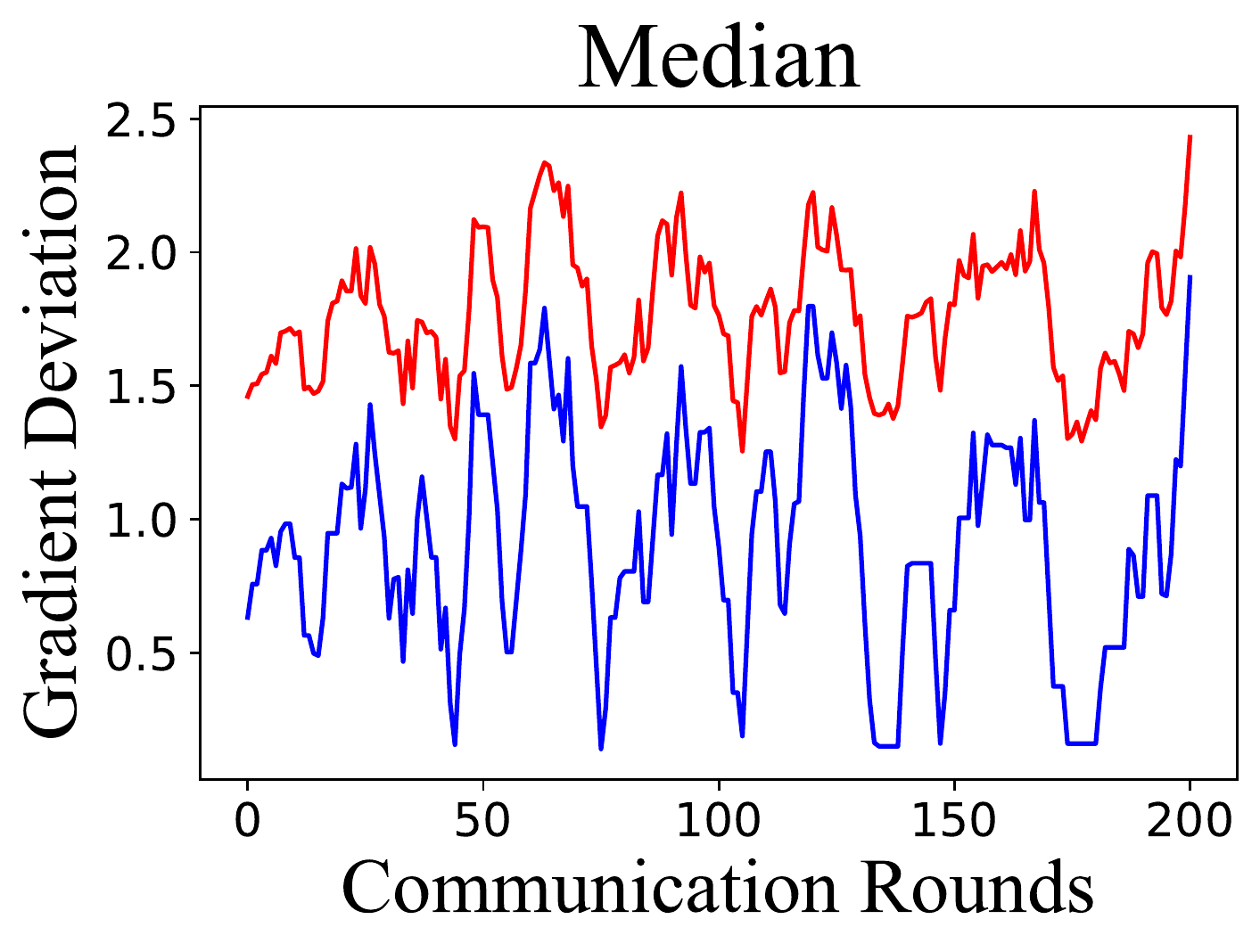}
\end{subfigure}
}
\scalebox{0.7}{
\begin{subfigure}[]{\textwidth}
\centering
\includegraphics[width=0.3\textwidth]{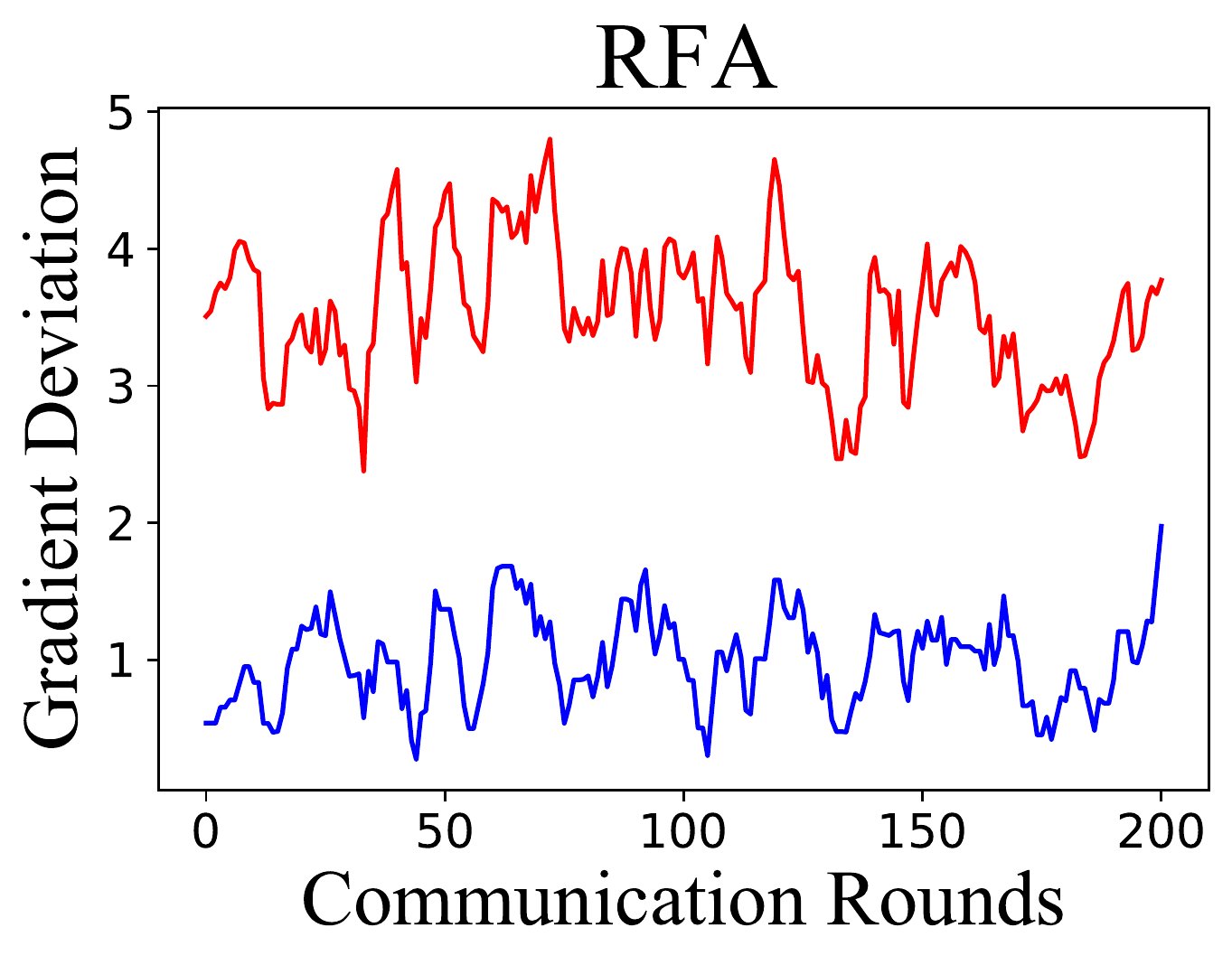}
\includegraphics[width=0.3\textwidth]{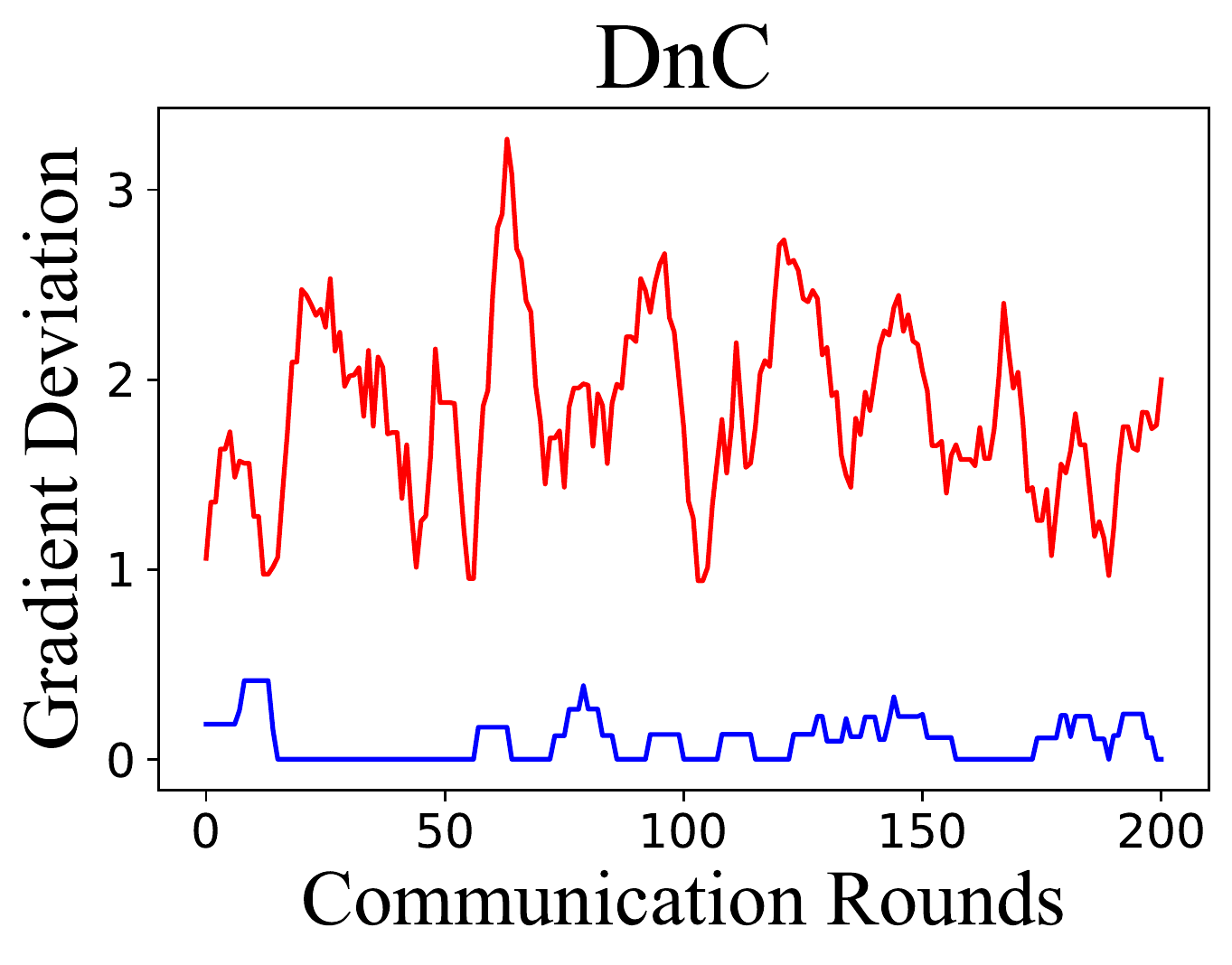}
\includegraphics[width=0.3\textwidth]{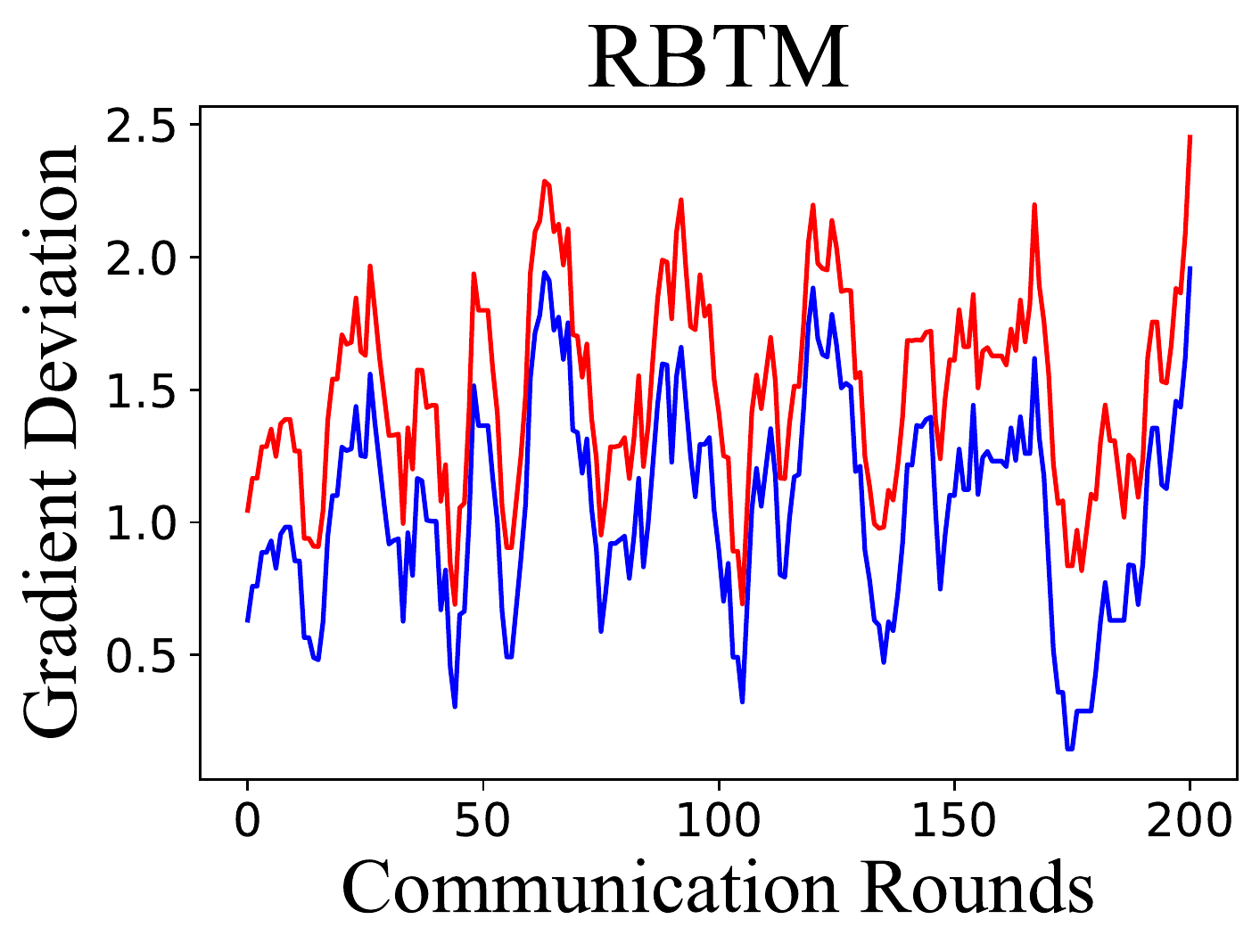}
\end{subfigure}
}
\caption{The gradient deviation $\|\aggsgrad{}-\avgsgrad{}\|$ of six different defenses w/ and w/o \proposedmethod under LIE attack on CIFAR-10. The lower the better.}
\label{fig:deviation}
\end{figure}

\end{document}